\documentclass{article}

\usepackage{microtype}
\usepackage{graphicx}
\usepackage{subcaption}
\usepackage{booktabs} 
\usepackage{kotex}
\usepackage{hyperref}
\usepackage{subcaption} 

\usepackage[accepted]{icml2026}
\usepackage{amsmath}
\usepackage{amssymb}
\usepackage{mathtools}
\usepackage{amsthm}
\usepackage{multirow}
\usepackage{enumitem}   
\usepackage{lipsum}
\usepackage{pifont}
\newcommand{\cmark}{\ding{51}}%
\newcommand{\xmark}{\ding{55}}%

\usepackage[capitalize,noabbrev]{cleveref}

\theoremstyle{plain}
\newtheorem{theorem}{Theorem}[section]
\newtheorem{proposition}[theorem]{Proposition}
\newtheorem{lemma}[theorem]{Lemma}

\theoremstyle{definition}

\newtheorem{assumption}[theorem]{Assumption}
\theoremstyle{remark}

\usepackage[textsize=tiny]{todonotes}

\usepackage{soul}
\usepackage{marginnote}

\newcounter{jccomment}
\todostyle{jcinline}{%
  inline,
  backgroundcolor=red!8!white,
  bordercolor=red!20!white,
  linecolor=red!20!white,
  inlinewidth=\linewidth,
  noinlinepar,
}
\todostyle{jcinlineshort}{%
  inline,
  backgroundcolor=red!8!white,
  bordercolor=red!20!white,
  linecolor=red!20!white,
  inlinewidth=0.5\linewidth,
  noinlinepar,
}
\todostyle{jcfancy}{%
  fancyline,
  backgroundcolor=red!12!white,
  bordercolor=red!20!white,
  linecolor=red!20!white
}

\makeatletter
\if@todonotes@disabled
  
\else
  
\fi
\makeatother

\icmltitlerunning{How Much Memory Do We Need? Adaptive Memory Gate for Neural Operators}

\begin{document}

\twocolumn[
  \icmltitle{How Much Memory Do We Need? Adaptive Memory Gate for Neural Operators}
  %



  \icmlsetsymbol{equal}{*}

  \begin{icmlauthorlist}
    \icmlauthor{Jihyeon Hur}{KAIST}
    \icmlauthor{Yongseok Kwon}{KAIST}
    \icmlauthor{Min-Gi Jo}{KAIST}
    \icmlauthor{Jeongwhan Choi}{KAIST}
    \icmlauthor{Noseong Park}{KAIST}
  \end{icmlauthorlist}

  \icmlaffiliation{KAIST}{KAIST, Republic of Korea}

  \icmlcorrespondingauthor{Noseong Park}{noseong@kaist.ac.kr}

  \icmlkeywords{Machine Learning, ICML}

  \vskip 0.3in
]



\printAffiliationsAndNotice{}  

\begin{abstract}
Neural operators have emerged as a powerful data-driven approach for solving time-dependent PDEs. Among recent advances, memory-augmented neural operators explicitly incorporate past states and have achieved remarkable performance under low-resolution observation settings. However, existing approaches apply a fixed memory weight regardless of observation conditions, such as resolution or physical parameters, limiting their adaptability. Our preliminary experiments reveal that optimal memory weight varies with resolution and viscosity, implying that a fixed memory weight cannot simultaneously optimize performance across diverse settings. We propose AMGFNO, which dynamically modulates memory weight through a learnable gate. On the Kuramoto-Sivashinsky and Burgers' equations, AMGFNO achieves 55--79\% nRMSE reduction over at low resolution, with the learned gate value automatically decreasing from $\bar{g} \approx 0.7$ to near-zero as resolution increases.
\end{abstract}
\section{Introduction}

Neural operators have emerged as a powerful data-driven alternative to classical numerical solvers for partial differential equations (PDEs)~\cite{li2021fourier,kovachki2023neural,lu2019deeponet}. Classical methods such as finite difference, finite volume, and finite element methods~\cite{FDM,FVM,FEM} suffer from high computational cost in high dimensions and sensitivity to initial conditions. Neural operators address these limitations by learning solution operators from data, enabling fast inference after a single training phase. Among the most practically important PDE families are time-dependent PDEs, which govern physical phenomena ranging from fluid dynamics to reaction-diffusion systems~\cite{takamoto2022pdebench}. For such PDEs, neural operators commonly adopt the Markovian assumption, predicting the next state only from the current state~\cite{FFNO,lippe2023pde}.

The Markovian assumption, however, breaks down when PDE solutions are observed at low resolution~\cite{S4FFNO}. As resolution decreases, the Nyquist-Shannon sampling theorem limits the number of observable Fourier modes, causing significant loss of high-frequency information. Mori-Zwanzig theory formalizes partial observation as inducing strongly non-Markovian behavior, making the memory term — accumulated information from past states — decisive for accurate prediction~\cite{Mori-Zwanzig}. 
\begin{table}[t]
    \centering
    \small
    \begin{tabular}{lccc}\toprule
        Method & Memory & Long-range & Adaptive \\\midrule
        FFNO                & \xmark & \xmark & \xmark \\
        Multi-Input FFNO    & \cmark & \xmark & \xmark  \\
        S4FFNO              & \cmark & \cmark & \xmark  \\
        \cmidrule(lr){1-4}
        \textbf{AMGFNO (Ours)}       & \cmark & \cmark & \cmark  \\
        \bottomrule
    \end{tabular}
    \caption{Comparison of neural operator architectures across three properties. Memory: incorporates past states beyond the current one; Long-range: spans the full trajectory rather than a fixed-length window; Adaptive: adjusts memory weight based on input conditions. AMGFNO is the first to achieve all three.}
    \label{tab:comp}
\end{table}

\begin{figure*}[h!]
    \centering
    \begin{subfigure}[t]{0.245\textwidth}
        \centering
        \includegraphics[width=\linewidth]{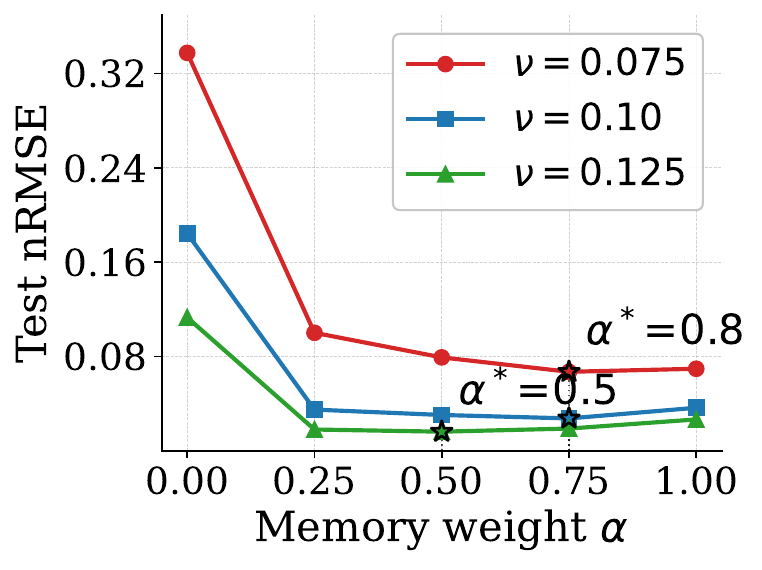}
        \caption{$f = 32$}
    \end{subfigure}\hfill
    \begin{subfigure}[t]{0.245\textwidth}
        \centering
        \includegraphics[width=\linewidth]{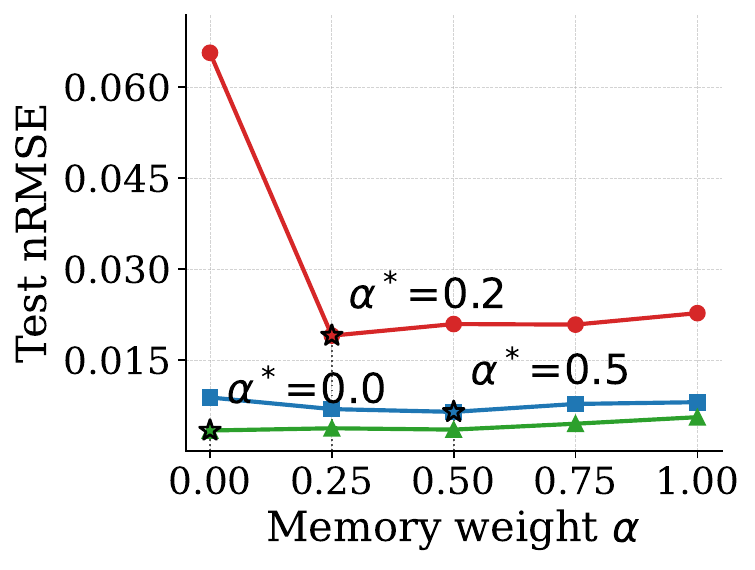}
        \caption{$f = 64$}
    \end{subfigure}\hfill
    \begin{subfigure}[t]{0.245\textwidth}
        \centering
        \includegraphics[width=\linewidth]{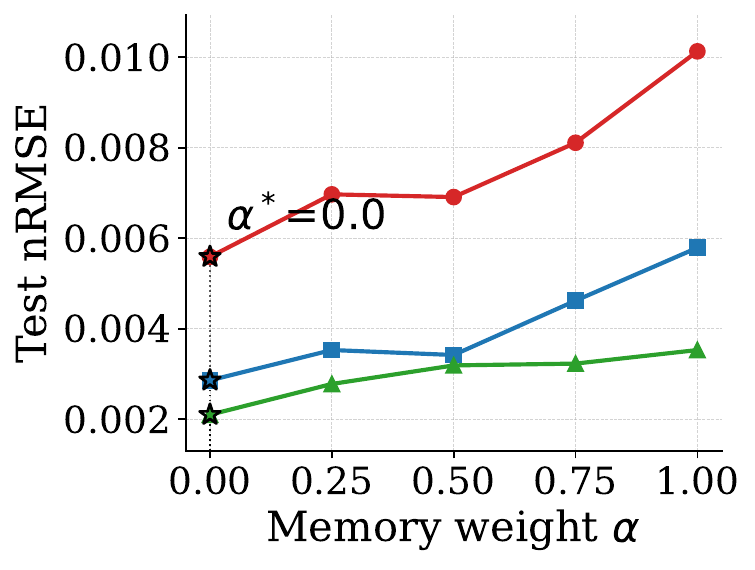}
        \caption{$f = 128$}
    \end{subfigure}\hfill
    \begin{subfigure}[t]{0.225\textwidth}
        \centering
        \includegraphics[width=\linewidth]{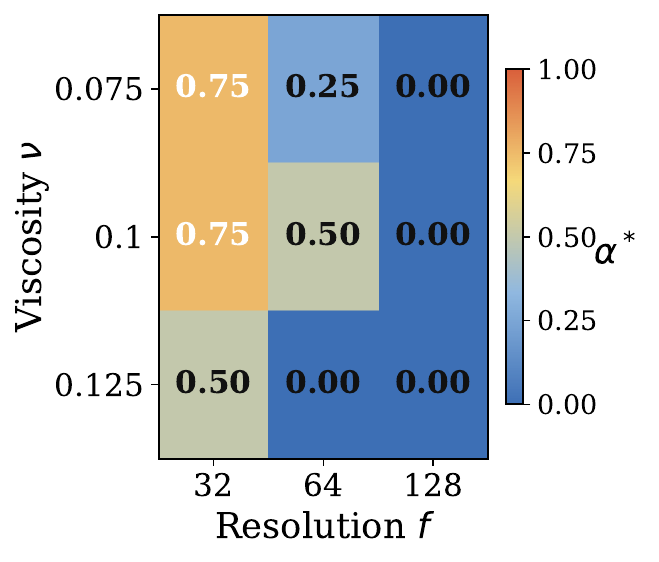}
        \caption{Optimal $\alpha^*$}
    \end{subfigure}
    \caption{Effect of memory weight $\alpha$ on nRMSE across resolutions $f$ and viscosity $\nu$ values on the KS equation. Panels (a-c) show nRMSE as a function of $\alpha$ for three viscosity values ($\nu \in \{0.075, 0.10, 0.125\}$) at a fixed resolution, with stars indicating optimal memory weight $\alpha^*$. Panel (d) summarizes the optimal memory weight $\alpha^*$ across all settings. 
    }
    \label{fig:motivation}
\end{figure*}

\paragraph{Fixed memory weight cannot adapt to varying conditions.} Explicitly modeling memory can therefore compensate for the information lost at low resolution. Existing memory-augmented neural operators use a fixed memory weight that controls how much memory is used, regardless of observation conditions~\cite{S4FFNO}. However, the utility of memory varies with resolution and physical parameters. Section~\ref{subsec:fixedalpha} presents a preliminary study showing that the optimal memory weight varies with resolution and viscosity, and that no single fixed memory weight simultaneously optimizes performance across all settings. 
\paragraph{Main idea.} We propose AdaptiveMemoryGateFNO (\textbf{AMGFNO}), a neural operator that dynamically controls memory weight via a adaptive memory gate. As summarized in Table~\ref{tab:comp}, FFNO uses no memory, Multi-Input FFNO captures only short-range history, and S4FFNO~\cite{S4FFNO} introduces long-range memory but lacks adaptive control. 
AMGFNO is the first to combine all three properties: (1) using past states as memory beyond the current timestep, (2) capturing long-range temporal dependencies across the full trajectory rather than a fixed window, and (3) adaptively adjusting memory weight based on the current input condition rather than using a fixed weight.
Throughout this paper, we use $\alpha$ to denote the manually-tuned memory weight in our ablation study, $\alpha^*$ for the optimal memory weight that minimizes error per condition, and $\bar{g}$ for AMGFNO's learned mean memory gate value.
The contributions of AMGFNO are summarized as follows:
\begin{itemize}[leftmargin=*, topsep=2pt]
    \item \textbf{Motivation.} We empirically show that the optimal memory weight $\alpha^*$ varies with resolution and viscosity (Section~\ref{subsec:fixedalpha}), and theoretically show that the gain of memory can converge to zero in the high-resolution limit (Section~\ref{sec:theory}).
    
    \item \textbf{Method.} We propose AMGFNO (Section~\ref{sec:method}), which dynamically modulates memory weight via an adaptive memory gate. The learned memory gate value decreases monotonically from $\bar{g} \approx 0.7$ to near-zero as resolution increases, consistent with theoretical predictions.
    
    \item \textbf{Results.} On the KS equation across three viscosity values and three resolutions, as well as Burgers' equation, AMGFNO achieves 55--79\% nRMSE reduction over the best fixed-memory baseline (S4FFNO) at low resolution ($f=32$) and 25--40\% at high resolution ($f=128$), demonstrating effective memory adaptation across all experimental settings (Section~\ref{sec:experiment}).
\end{itemize}






\section{Preliminaries and Motivation}
\label{sec:preliminary}
\subsection{Time-dependent PDE Problem Definition}
\label{sec:pde-definition}

We consider a time-dependent partial differential equation defined on a spatial domain
$\Omega\subset\mathbb R^d$ and a time interval $[0,T]$:
\begin{equation}
    \frac{\partial u}{\partial t}(t,x)
    =
    \mathcal L_{\rho}[u](t,x),
    \qquad
    (t,x)\in[0,T]\times\Omega.
    \label{eq:general-pde}
\end{equation}
The PDE is equipped with an initial condition
\begin{equation}
    u(0,x)=u_0(x),
    \qquad
    x\in\Omega,
    \label{eq:general-ic}
\end{equation}
and a boundary condition
\begin{equation}
    \mathcal B[u|_{\partial\Omega}](t)=0,
    \qquad
    t\in[0,T].
    \label{eq:general-bc}
\end{equation}
Here $u(t,\cdot):\Omega\to\mathbb R^V$ denotes the PDE state at time $t$, $x$ is the spatial coordinate, V is the channel dimension of the input feature space, $\mathcal L_{\rho}$ is a differential operator in $x$ parameterized by physical parameters $\rho$, and $\mathcal B$ specifies the boundary condition. In our one-dimensional experiments, we use periodic boundary conditions.
Let $\Phi_t$ denote the solution map induced by the PDE, so that
\[
    u(t,\cdot)=\Phi_t(u_0).
\]
The learning problem considered in this paper is to approximate the time evolution of this solution map from discretized observations of the PDE state. The resolved-state and spectral-truncation notation used in the theory is introduced in Section~\ref{sec:theory}. Dataset-specific PDE equations, parameter values, and data-generation details are deferred to Appendix~\ref{app:exp-details}.

\subsection{Preliminary Study: How Much Memory Should Be Used?}
\label{subsec:fixedalpha}
To analyze when memory is beneficial, we conduct an ablation study 
that manually controls the memory weight in S4FFNO~\cite{S4FFNO}. 
S4FFNO follows a \textsc{ffno--ffno--s4--ffno--ffno} layer structure, 
where the S4 memory layer is inserted at the middle position 
($\ell^* = 3$). Let $\mathbf{h}_t^{(\ell^*-1)} \in \mathbb{R}^{N \times H}$ 
denote the hidden state at the output of the second FFNO layer, and 
let $\mathbf{z} = \mathcal{M}(\mathbf{u}_t) \in \mathbb{R}^{N \times H}$ 
denote the output of the S4 memory branch. We replace the original 
residual fusion with a convex combination that explicitly 
controls the relative weight of memory:
%
\begin{equation}
    \mathbf{h}_{t+1} \leftarrow \alpha \cdot \mathbf{z} + (1-\alpha) \cdot \mathbf{h}_t,
    \qquad \alpha \in [0, 1],
\end{equation}
where the updated $\mathbf{h}_{t+1}$ is passed to the third FFNO layer.
Here, $\alpha$ is the memory weight: $\alpha = 0$ 
corresponds to a pure Markovian model that bypasses the memory 
branch, $\alpha = 1$ fully replaces the residual with the memory 
output. We sweep $\alpha \in \{0, 0.25, 0.5, 0.75, 1.0\}$ on the Kuramoto--Sivashinsky (KS) equation across three spatial resolutions ($f \in \{32, 64, 128\}$) and three viscosity values ($\nu \in \{0.075, 0.10, 0.125\}$). Results are shown in Figure~\ref{fig:motivation}. The same observations hold under an additive parameterization $\mathbf{h}_{t+1} \leftarrow \alpha \cdot \mathbf{z} + \mathbf{h}_t$; see Appendix~\ref{app:additive_motivation} for details.


\paragraph{Observation 1: Optimal memory weight varies significantly 
across conditions.}
Figure~\ref{fig:motivation}(a-c) shows test nRMSE as a function of 
$\alpha$ for different resolutions and viscosities, with stars marking 
the optimal memory weight $\alpha^*$. At low resolution ($f=32$), using optimal 
$\alpha^* \in \{0.5, 0.75\}$ reduces error by up to 41\% compared to pure memory ($\alpha = 1$): from 0.070 to 0.067 at $\nu = 0.075$, from 
0.037 to 0.028 at $\nu = 0.10$, and from 0.027 to 0.016 at 
$\nu = 0.125$. At high resolution ($f=128$), $\alpha^*=0$ across all 
viscosities. At intermediate resolution ($f=64$), $\alpha^*$ varies 
from $0$ to $0.5$ depending on viscosity. As the heatmap in 
Figure~\ref{fig:motivation}(d) shows, optimal $\alpha^*$ spans 
$\{0.0, 0.25, 0.5, 0.75\}$ across the 9 settings---no single fixed 
value is optimal, motivating adaptive memory gate. This trend holds 
under both convex and additive fusion (Appendix~\ref{app:additive_motivation}).

\paragraph{Observation 2: Fixed $\alpha$ fails to generalize across setting.}
As the heatmap in Figure~\ref{fig:motivation}(d) shows, a fixed 
$\alpha$ optimized for one condition performs poorly in others. For 
instance, $\alpha=0.75$ works well at $f=32$ (high viscosity) but 
degrades performance at $f=128$, while $\alpha=0$ underexploits memory 
at low resolution. At intermediate resolution ($f=64$, 
Figure~\ref{fig:motivation}(b)), optimal $\alpha^*$ varies widely---from 
$0$ to $0.5$ across viscosities---illustrating that fine-grained, 
condition-dependent control is necessary.

\paragraph{Motivation.}
These findings reveal that memory weight should adapt dynamically to the simulation condition. A fixed $\alpha$ treats memory as a static architectural choice, ignoring that optimal memory usage depends on the observation setting (e.g., resolution, viscosity). We propose an 
Adaptive Memory Gate Block in Section~\ref{subsec:gate} that learns to modulate memory based on the current hidden state, enabling the model to automatically adjust across different observation  settings without manual tuning.





\subsection{Theoretical Motivation}
\label{sec:theory}

\paragraph{Intuition.} Memory should be treated as a condition-dependent resource. At low
observation resolution, unresolved Fourier modes may still influence the
next resolved state, making the observed dynamics effectively
non-Markovian and temporal history informative. As the resolution
increases, this projection-induced hidden influence diminishes; a fixed
memory path then has decreasing approximation benefit while still
contributing additional finite-sample complexity.

\paragraph{Setup.} We formalize this mechanism under the ideal Fourier projection analyzed
in Appendix~\ref{app:theory}. Let $m_f:=\lfloor f/2\rfloor$, let $P_f$
denote projection onto Fourier modes $|n|\le m_f$, and let $Q_f=I-P_f$
denote the unresolved tail. For a trajectory
$u_{t+1}^{\rho}=\Phi_{\Delta t}^{\rho}(u_t^\rho)$ under physical
condition $\rho$, write
\[
\begin{gathered}
    u_{t,f}^{\rho}:=P_f u_t^\rho, \quad
    Y_t^{f,\rho}:=P_fu_{t+1}^{\rho}, \\
    \mathcal U_t^{f,\rho}:=(u_{0,f}^{\rho},\ldots,u_{t,f}^{\rho}).
\end{gathered}
\]
Here $u_{t,f}^{\rho}$ is the resolved state,
$\mathcal U_t^{f,\rho}$ is the resolved history, and
$Y_t^{f,\rho}$ is the one-step resolved target. The influence of the
unresolved tail $Q_f u_t^\rho$ on the next resolved state is measured by
\begin{equation}
    \delta_{t,f,\rho}
    :=
    P_f\Phi_{\Delta t}^{\rho}(u_t^\rho)
    -
    P_f\Phi_{\Delta t}^{\rho}(P_f u_t^\rho).
    \label{eq:main-projection-residual}
\end{equation}
Under the finite-horizon regularity and local Lipschitz assumptions in
Appendix~\ref{app:theory},
\begin{equation}
    \mathbb E\|\delta_{t,f,\rho}\|^2
    \le
    \kappa_{{\rm tail},\rho}f^{-2\beta}.
    \label{eq:main-residual-decay}
\end{equation}
Thus, as $f$ grows, the resolved dynamics become increasingly close to a
Markovian system in the current resolved state.

\paragraph{From residual decay to memory-gain bounds.}
To connect this residual decay to neural-operator approximation, let
$\mathcal T_{f,\rho}(\phi)=P_f\Phi_{\Delta t}^{\rho}(\phi)$ denote the
projected Markovian target, and let
$b_{f,\rho}(\mathcal U_t^{f,\rho})$ denote the history-conditioned Bayes
target. We write $\mathcal E^{\mathcal C}_{f,\rho}$ for the
approximation error of a model class $\mathcal C$ to this Bayes target,
and $\mathcal A^F_{f,\rho}$ for FFNO's approximation error to the
projected Markovian target. Formal definitions are given in
Appendix~\ref{app:theory}.

The distinction between $\mathcal A^F_{f,\rho}$ and
$\mathcal E^F_{f,\rho}$ is important. The assumption is imposed only on
the projected Markovian target $\mathcal T_{f,\rho}$ through
$\mathcal A^F_{f,\rho}$; it does not assume that FFNO can approximate the
history-conditioned Bayes target. Theorem~\ref{thm:gain-convergence}
uses the residual decay in Eq.~\eqref{eq:main-residual-decay} to transfer
projected-target approximation into a bound on the actual Bayes-target
error $\mathcal E^F_{f,\rho}$.

Define the memory gains over FFNO by
\begin{equation}
    \Gamma^{S4}_{f,\rho}
    :=
    \mathcal E^F_{f,\rho}
    -
    \mathcal E^{S4}_{f,\rho},
    \qquad
    \Gamma^{AMG}_{f,\rho}(\lambda)
    :=
    \mathcal E^F_{f,\rho}
    -
    \mathcal E^{AMG}_{f,\rho}(\lambda).
    \label{eq:main-memory-gains}
\end{equation}

\begin{theorem}[Convergence of projection-induced memory gain]
\label{thm:gain-convergence}
Under the assumptions~\ref{ass:wellposed-sobolev}--\ref{ass:ffno-projected-approx},
\begin{equation}
    \mathcal E^F_{f,\rho}
    \le
    2\mathcal A^F_{f,\rho}
    +
    2\kappa_{{\rm tail},\rho}f^{-2\beta}.
    \label{eq:main-ffno-bayes-bound}
\end{equation}
Consequently, $\mathcal E^F_{f,\rho}\to0$ as $f\to\infty$. Since the
memory-augmented classes contain the FFNO fallback,
\begin{equation}
    0
    \le
    \Gamma^{S4}_{f,\rho}
    \le
    \mathcal E^F_{f,\rho},
    \qquad
    0
    \le
    \Gamma^{AMG}_{f,\rho}(\lambda)
    \le
    \mathcal E^F_{f,\rho}
    \label{eq:main-gain-convergence}
\end{equation}
for every $\lambda\ge0$. Hence the additional approximation gain of
memory induced by spectral projection vanishes in the high-resolution
limit.
\end{theorem}
The proof is given in Appendix~\ref{app:proof-thm21}.

AMGFNO is designed to track this convergence through its frequency-aware
gate. For a full reference state $\phi$ with Fourier coefficients
$a_n(\phi)$, define the unresolved spectral-energy ratio
\begin{equation}
    \omega_f(\phi)
    :=
    \frac{
        \sum_{|n|>m_f}|a_n(\phi)|^2
    }{
        \sum_{n\in\mathbb Z}|a_n(\phi)|^2
    },
    \label{eq:main-omega-state}
\end{equation}
with value zero when the denominator is zero. For condition $(f,\rho)$,
define the population unresolved-energy ratio
\begin{equation}
    \omega^{\rm pop}_{f,\rho}
    :=
    \mathbb E_{u_0,t}\left[\omega_f(u_t^\rho)\right].
    \label{eq:main-omega-pop}
\end{equation}
In experiments, $\omega^{\rm pop}_{f,\rho}$ is replaced by the empirical
training-set estimate $\widehat\omega_{f,\rho}$, computed from
high-resolution training trajectories and held fixed for the condition
$(f,\rho)$. This quantity is used only as a condition-level calibration
constant; no high-resolution test information is used.

The scalar frequency gate is
\begin{equation}
    \chi_{f,\rho}
    :=
    \begin{cases}
    \sigma\!\left(w_1\log \omega^{\rm pop}_{f,\rho}+b_\omega\right),
        & \omega^{\rm pop}_{f,\rho}>0,\\
    0,
        & \omega^{\rm pop}_{f,\rho}=0,
    \end{cases}
    \qquad
    w_1>0.
    \label{eq:main-frequency-gate}
\end{equation}
Lemma~\ref{lem:gate-convergence-final} shows that
$\omega^{\rm pop}_{f,\rho}\to0$, and therefore
$\chi_{f,\rho}\to0$. Since the content gate is bounded entrywise by one,
the induced AMGFNO memory budget is bounded by
\begin{equation}
    \lambda_{f,\rho}:=\chi_{f,\rho}^2,
    \label{eq:main-amg-budget}
\end{equation}
which also converges to zero.

Finally, the adaptive memory gate also controls finite-sample complexity. Let
$R_m(\mathfrak L)$ denote the Rademacher complexity of a loss class, and
let $U^F_{f,\rho}$, $U^{S4}_{f,\rho}$, and
$U^{AMG}_{f,\rho}(\lambda)$ denote the ERM upper bounds defined in
Appendix~\ref{app:theory-finite-sample}.

\begin{theorem}[Adaptive high-resolution comparison]
\label{thm:main-amg-highres}
Let $\lambda_{f,\rho}=\chi_{f,\rho}^2$. Under the 
assumptions~\ref{ass:wellposed-sobolev}--\ref{ass:trajectory-sampling}
and the sufficient conditions of
Proposition~\ref{prop:amg-complexity-convergence},
\begin{equation}
    R_m(\mathfrak L^{AMG}_{f,\rho}(\lambda_{f,\rho}))
    -
    R_m(\mathfrak L^F_{f,\rho})
    =
    O(\sqrt{\lambda_{f,\rho}})
    \to0.
    \label{eq:main-amg-complexity-convergence}
\end{equation}
Thus,
\begin{equation}
    U^{AMG}_{f,\rho}(\lambda_{f,\rho};\eta)
    -
    U^F_{f,\rho}(\eta)
    \to0.
\end{equation}
If fixed-fusion S4FFNO retains a positive asymptotic
Rademacher-complexity gap over FFNO, then for all sufficiently large
$f$,
\begin{equation}
    U^{AMG}_{f,\rho}(\lambda_{f,\rho};\eta)
    <
    U^{S4}_{f,\rho}(\eta).
\end{equation}
\end{theorem}
The proof is given in Appendix~\ref{app:proof-thm22}.

The analysis isolates the memory benefit induced by projection-based loss
of Fourier modes. It does not rule out persistent memory gains from other
sources of partial observability, such as latent parameters, stochastic
forcing, measurement noise, or temporal subsampling. The theory controls
one-step resolved prediction risk, while accumulated autoregressive
rollout error is evaluated empirically in Section~\ref{sec:experiment}.




\begin{figure*}[t]
    \centering
    \includegraphics[width=1\textwidth]{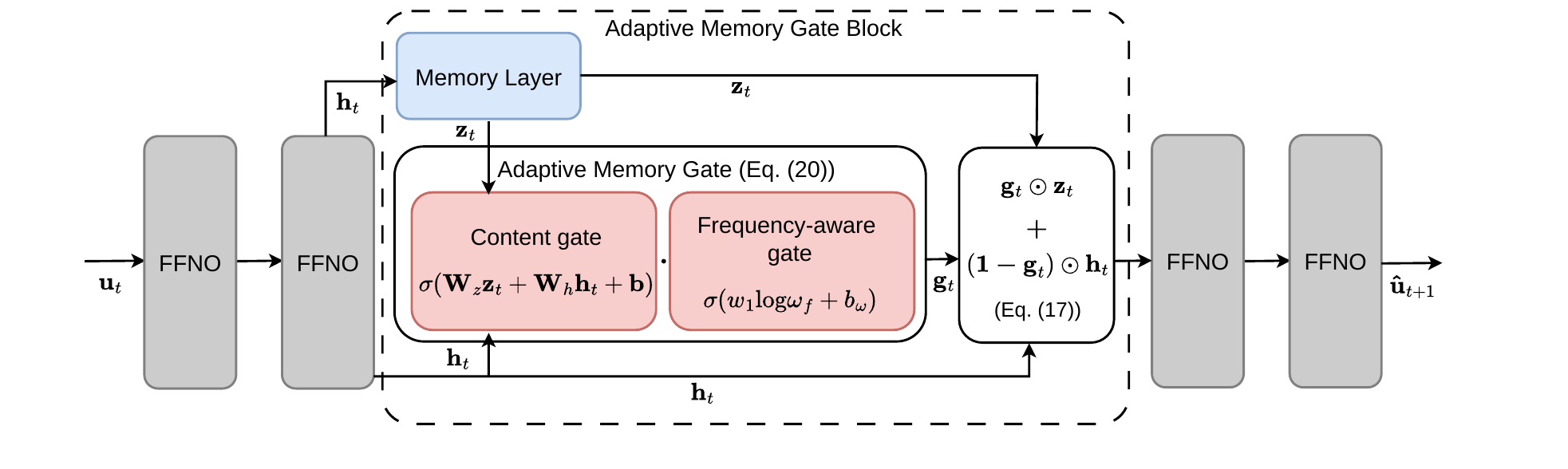}
    \caption{Overview of the AMGFNO architecture. The model consists of five layers 
with an Adaptive Memory Gate (AMG) Block at the middle layer ($\ell=3$). 
The AMG Block combines the memory layer output $\mathbf{z}_t$ and 
the Markovian state $\mathbf{h}_t$ (output of the second FFNO layer, $\ell=2$) using an adaptive memory gate $\mathbf{g}_t$ 
that is decomposed into two components: a content gate 
$\sigma(\mathbf{W}_z\mathbf{z}_t + \mathbf{W}_h\mathbf{h}_t + \mathbf{b})$ 
for input-adaptive control and a frequency-aware gate 
$\sigma(w_1 \log \omega_f + b_\omega)$ for conditional priors. 
The final output is computed as 
$\mathbf{g}_t \odot \mathbf{z}_t + (1-\mathbf{g}_t) \odot \mathbf{h}_t$.}
    \label{fig:overview}
\end{figure*}

\section{Proposed Method}\label{sec:method}

AMGFNO builds on the FFNO backbone by inserting an Adaptive Memory Gate (AMG) Block at the middle layer, as illustrated in Figure~\ref{fig:overview}. The AMG Block consists of a memory layer and an adaptive memory gate (AMG). The AMG consists of content gate and frequency-aware gate. The input PDE state 
$\mathbf{u}_t\in \mathbb{R}^{S \times V}$ is first encoded into a 
latent representation 
$\mathbf{h}_t^{(0)} = \mathrm{Enc}(\mathbf{u}_t, \mathbf{s}) 
\in \mathbb{R}^{S \times H}$, where $\mathbf{s} \in \mathbb{R}^{S 
\times 1}$ is the spatial grid and $H$ is the hidden dimension. The 
latent state then passes through $L=5$ layers in the order 
\textsc{ffno--ffno--amg--ffno--ffno}:
\begin{equation}
    \mathbf{h}_t^{(\ell)} =
    \begin{cases}
        \mathcal{F}_\ell\!\left(\mathbf{h}_t^{(\ell-1)}\right)
            + \mathbf{h}_t^{(\ell-1)},
        & \ell \neq \ell^{*} \\[6pt]
        \mathbf{g}_t \odot \mathbf{z}_t
            + \left(1 - \mathbf{g}_t\right) \odot \mathbf{h 
            }_t^{(\ell-1)},
        & \ell = \ell^{*}
    \end{cases}
    \label{eq:layer}
\end{equation}
where $\ell^{*}=3$ is the index of the Adaptive Memory Gate (AMG) Block and $\mathbf{z}_t = \mathcal{M}(\mathbf{h}_{1:t}^{(\ell^{*}-1)})$ is the memory layer output. At $\ell \neq \ell^{*}$, each layer applies an FFNO operation with a residual connection. At $\ell = \ell^{*}$, the AMG Block combines the memory output $\mathbf{z}_t$ and the Markovian hidden state $\mathbf{h}_t^{(\ell^*-1)}$ via the adaptive memory gate $\mathbf{g}_t \in [0,1]^{S \times H}$, which is computed from the current feature state. The gate dynamically adjusts memory weight across different simulation conditions, allowing the model to modulate memory usage based on the learned representation. The final prediction is $\hat{\mathbf{u}}_{t+1} = \mathrm{Dec}(\mathbf{h}_t^{(L)})$.
\subsection{Factorized Fourier Neural Operator (FFNO) Layer}
Each FFNO layer $\mathcal{F}_\ell$ operates along the \emph{spatial} axis and combines a spectral path with a local linear path:
\begin{equation}
    \mathcal{F}_\ell(\mathbf{h})
    = \mathcal{F}^{-1}\!\left(
        \mathbf{W}_\ell^{\mathrm{freq}} \cdot
        \left[\mathcal{F}(\mathbf{h})\right]_{1:K}
      \right)
    + \mathbf{W}_\ell^{\mathrm{loc}}\,\mathbf{h},
\end{equation}
where $\mathcal{F}$ is the discrete Fourier transform along the spatial dimension, $K$ is the number of retained Fourier modes, and $\mathbf{W}_\ell^{\mathrm{freq}},\, \mathbf{W}_\ell^{\mathrm{loc}}$ are learnable weight matrices.

\subsection{Memory Layer}
The memory output $\mathbf{z}_t = \mathcal{M}(\mathbf{h}_{1:t})$ is computed by applying S4 (Structured State Space Model) along the \emph{temporal} axis. For each spatial position $p \in \{1,\dots,S\}$ independently:
\begin{equation}
    \mathbf{v}_t^{(p)} = \mathbf{A}\,\mathbf{v}_{t-1}^{(p)} + \mathbf{B}\,\mathbf{h}_t^{(p)},
    \qquad
    \mathbf{z}_t^{(p)} = \mathbf{C}\,\mathbf{v}_t^{(p)},
\end{equation}
where $\mathbf{v}_t^{(p)} \in \mathbb{R}^{N}$ is the S4 internal state that accumulates temporal context up to time $t$, and $\mathbf{A},\, \mathbf{B},\, \mathbf{C}$ are learnable parameters initialized with HiPPO~\cite{gu2020hippo}. The spatial dimension is treated as a batch axis, so all positions share the same SSM parameters.

\subsection{Adaptive Memory Gate}
\label{subsec:gate}
As shown in Figure~\ref{fig:overview}, the adaptive memory gate $\mathbf{g}_t \in (0,1)^{S \times H}$ controls how much the memory output $\mathbf{z}_t$ contributes relative to the current hidden state $\mathbf{h}_t$. The gate consists of a content gate and a frequency-aware gate:
\begin{equation}
    \mathbf{g}_t
    = \underbrace{
        \sigma\!\left(
            \mathbf{W}_z\,\mathbf{z}_t
            + \mathbf{W}_h\,\mathbf{h}_t
            + \mathbf{b}
        \right)
      }_{\text{content gate} \;\in\;(0,1)^{S\times H}}
    \;\cdot\;
    \underbrace{
        \sigma\!\left(w_1 \log \omega_f + b_\omega\right)
      }_{\text{frequency-aware gate} \;\in\;(0,1)}.
\end{equation}

\paragraph{Content gate.}
The content gate $\sigma(\mathbf{W}_z \mathbf{z}_t + \mathbf{W}_h \mathbf{h}_t + \mathbf{b}) \in (0,1)^{S \times H}$ provides input-adaptive control that varies per sample, channel, and spatial location. It takes the memory output $\mathbf{z}_t$ and the current hidden state $\mathbf{h}_t$ as input and learns to modulate memory weight based on their learned representations. The parameters $\mathbf{W}_z$, $\mathbf{W}_h$, and $\mathbf{b}$ are initialized to zero, allowing the content gate to learn data-driven corrections during training while deferring to the frequency-aware prior at initialization.

\paragraph{Frequency-aware gate.}
The frequency-aware gate $\sigma(w_1 \log \omega_f + b_\omega) \in (0,1)$ supplies a conditional prior based on the observation condition $(f, \rho)$. Here $\omega_f$ is the high-frequency energy ratio defined in Eq.~\eqref{eq:main-omega-pop}, computed as the mean over the training set and held fixed during training and inference. We found empirically that this fixed $\omega_f$ estimate consistently outperforms per-batch computation by 5--20\% across settings (see Appendix~\ref{app:omega_ablation}), suggesting that $\omega_f$ serves best as a calibration constant characterizing the observation condition rather than a dynamic per-sample feature. The parameters $w_1 = 2.0$ and $b_\omega = 3.8$ are analytically initialized from two anchor conditions ($\omega_{32}$ and $\omega_{128}$) to provide a resolution-aware prior that the content gate refines during training.

\begin{table*}[th!]
\centering
\resizebox{\textwidth}{!}{%
\begin{tabular}{cccccccccccccc}
\toprule
\multirow{3}{*}{Architecture} & \multirow{3}{*}{Uses memory} 
& \multicolumn{9}{c}{Kuramoto--Sivashinsky} 
& \multicolumn{3}{c}{Burgers'} \\
\cmidrule(lr){3-11}\cmidrule(lr){12-14}
& & \multicolumn{3}{c}{$f=32$} 
& \multicolumn{3}{c}{$f=64$} 
& \multicolumn{3}{c}{$f=128$}
& \multirow{2}{*}{$f=32$} & \multirow{2}{*}{$f=64$} & \multirow{2}{*}{$f=128$} \\
\cmidrule(lr){3-5}\cmidrule(lr){6-8}\cmidrule(lr){9-11}
& & $\nu{=}.075$ & $\nu{=}.10$ & $\nu{=}.125$ 
  & $\nu{=}.075$ & $\nu{=}.10$ & $\nu{=}.125$
  & $\nu{=}.075$ & $\nu{=}.10$ & $\nu{=}.125$
  & & & \\
\midrule
Factformer (1D) & \xmark 
  & 0.436 & 0.391 & 0.149
  & 0.195 & 0.086 & 0.022
  & 0.058 & 0.030 & 0.017
  & 0.190 & 0.162 & 0.117 \\
GKT & \xmark 
  & 0.588 & 0.601 & 0.314
  & 0.401 & 0.120 & 0.016
  & 0.028 & 0.013 & 0.007
  & 0.356 & 0.349 & 0.307 \\
U-Net & \xmark 
  & 0.542 & 0.511 & 0.249
  & 0.147 & 0.062 & 0.022
  & 0.033 & 0.027 & 0.014
  & 0.188 & 0.171 & 0.112 \\
FFNO & \xmark 
  & 0.500 & 0.446 & 0.187
  & 0.107 & 0.033 & 0.004
  & \textbf{0.006} & 0.004 & \textbf{0.002}
  & 0.207 & 0.146 & 0.099 \\
Multi Input FFNO & \cmark 
  & 0.364 & 0.308 & 0.092
  & 0.108 & 0.046 & 0.005
  & 0.057 & 0.052 & 0.023
  & 0.099 & 0.054 & 0.028 \\
S4FFNO & \cmark 
  & 0.139 & 0.108 & 0.031
  & 0.036 & 0.011 & 0.004
  & 0.008 & 0.005 & 0.003
  & 0.053 & 0.037 & 0.030 \\
\textbf{AMGFNO (Ours)} & \cmark 
  & \textbf{0.062} & \textbf{0.023} & \textbf{0.014}
  & \textbf{0.021} & \textbf{0.007} & \textbf{0.003}
  & \textbf{0.006} & \textbf{0.003} & \textbf{0.002}
  & \textbf{0.024} & \textbf{0.018} & \textbf{0.020} \\
\bottomrule
\end{tabular}
}
\caption{nRMSE values at different resolutions for KS and Burgers' equations. 
AMGFNO consistently outperforms all baselines across resolutions, achieving up to 79\% 
nRMSE reduction over S4FFNO at $f=32$ while suppressing memory at $f=128$ to match 
or surpass FFNO.}
\label{tab:main}
\end{table*}

\section{Experiments}
\label{sec:experiment}

\begin{table}[t]
\centering
\caption{Converged mean gate value $\bar{g}$ (mean of last 20 epochs) for AMGFNO across different resolution and viscosity settings.}
\label{tab:omega_gate}
\footnotesize
\begin{tabular}{lcccccc}
\toprule
\multirow{2}{*}{Dataset} & \multirow{2}{*}{$\nu$} & \multicolumn{3}{c}{$\bar{g}$ by Resolution $f$} \\
\cmidrule(lr){3-5}
& & $f=32$ & $f=64$ & $f=128$ \\
\midrule
\multirow{3}{*}{KS} 
 & 0.075 & 0.6915 & 0.4720 & 0.0030 \\
 & 0.10   & 0.7036 & 0.2150 & 0.0008 \\
 & 0.125 & 0.6976 & 0.0721 & 0.0002 \\
\midrule
Burgers' & 0.001 & 0.0267 & 0.0165 & 0.0020 \\
\bottomrule
\end{tabular}
\end{table}

\subsection{Experimental Setup}
\label{sec:setup}

\paragraph{Benchmarks and observation settings.}
We evaluate on two 1D time-dependent PDEs with periodic boundary conditions.
For the Kuramoto--Sivashinsky (KS) equation, we consider three viscosity values $\nu \in \{0.075, 0.1, 0.125\}$ to span different spectral conditions.
Lower viscosity leads to more chaotic dynamics with significant high-frequency content, while higher viscosity smooths the solution.
For Burgers' equation, we use viscosity $\nu = 0.001$, which maintains strong nonlinear advection.
To assess how memory usage should adapt across observation conditions, we consider three spatial resolutions $f\in\{32,64,128\}$, representing low-, medium-, and high-resolution, respectively.
At low resolution, many high-frequency Fourier modes are unobserved; at high resolution, most spectral information is directly observable.
PDE-specific parameters, final simulation times, and data generation procedures are detailed in Appendix~\ref{app:exp-details}.

\paragraph{Training objective and procedure.}
All models are trained with a one-step prediction objective.
Given the current resolved state $u_j^{f,(i)}$ on spatial grid $S_f$ at time $t_j$, the model predicts the next state $u_{j+1}^{f,(i)}$ at time $t_{j+1}$.
During training, we employ teacher forcing: each step receives the ground-truth state as input rather than the model's own prediction from the previous step.
For memory-augmented models, the internal memory state is updated using the same teacher-forced sequence.
The training loss is the mean squared error averaged over mini-batches $\mathcal{B}$ and all timesteps within each trajectory:
\begin{equation}
    \mathcal{L}_{\mathrm{train}}(\theta)
    =
    \frac{1}{|\mathcal{B}|J}
    \sum_{i\in\mathcal{B}}
    \sum_{j=0}^{J-1}
    \left\|
        \mathcal{G}_\theta(u_j^{f,(i)}, g_f)
        -
        u_{j+1}^{f,(i)}
    \right\|_2^2,
\end{equation}
where $\mathcal{G}_\theta$ is the neural operator, $g_f$ is the grid encoding, and $J$ is the number of timesteps per trajectory.
All models use the same optimizer settings and are trained for 200 epochs.

\paragraph{Evaluation protocol.}
At test time, models are evaluated via autoregressive rollout.
Only the initial state $u_0^{f,(i)}$ is provided as ground truth; subsequent predictions are generated by repeatedly feeding the model's own output back as input:
\begin{equation}
\begin{aligned}
    \widetilde{u}_0^{f,(i)} &= u_0^{f,(i)}, \\
    \widetilde{u}_{j+1}^{f,(i)} &= \mathcal{G}_\theta(\widetilde{u}_j^{f,(i)}, g_f), \quad j=0,\ldots,J-1.
\end{aligned}
\end{equation}
This autoregressive evaluation measures the model's ability to maintain accurate predictions over long rollouts, where one-step errors can accumulate.
Performance is measured by the normalized root mean squared error (nRMSE), averaged over all test trajectories and rollout timesteps:
\begin{equation}
    \mathrm{nRMSE}(f)
    =
    \frac{1}{N_{\mathrm{test}}J}
    \sum_{i=1}^{N_{\mathrm{test}}}
    \sum_{j=1}^{J}
    \frac{
    \left\|\widetilde{u}_j^{f,(i)} - u_j^{f,(i)}\right\|_2
    }{
    \left\|u_j^{f,(i)}\right\|_2
    },
\end{equation}
where $N_{\mathrm{test}}$ is the number of test trajectories.
Lower nRMSE indicates more accurate trajectory prediction.

\paragraph{Baseline comparisons.}
We compare AMGFNO against two categories of baselines.
Markovian neural operators use only the current state without temporal memory: Factformer~\cite{factformer}, Galerkin Transformer(GKT)~\cite{cao2021choose}, U-Net~\cite{unet_pde},and  FFNO~\cite{FFNO}.
Memory-augmented operators incorporate information from past states: Multi-Input FFNO~\cite{S4FFNO}, which concatenates the previous four timesteps as input, and S4FFNO~\cite{S4FFNO}, which uses a state-space memory layer with fixed fusion.
All baselines use comparable model capacity and are trained with identical hyperparameters for fair comparison.
Architectural details and implementation specifics are provided in Appendix~\ref{app:exp-details}.



\subsection{Main Result}
\label{subsec:main_result}
Table~\ref{tab:main} shows nRMSE results for the KS equation across three viscosity values and three resolutions, as well as for Burgers' equation.
AMGFNO consistently outperforms all baselines across all resolutions and viscosities, with the largest performance gap at low resolution and the smallest at high resolution.

\paragraph{Low resolution ($f=32$).}
AMGFNO significantly outperforms all baselines across all viscosities at resolution 32.
Compared to the previous best-performing model S4FFNO, nRMSE on KS decreases from 0.139 to 0.062 at $\nu=0.075$, from 0.108 to 0.023 at $\nu=0.1$, and from 0.031 to 0.014 at $\nu=0.125$.
On Burgers', nRMSE decreases from 0.053 to 0.024 compared to S4FFNO.

AMGFNO improves over baselines by adaptively adjusting the blend between the current state and temporal memory. At low resolution, the gate increases memory weight to compensate for unobserved high-frequency modes, whereas S4FFNO uses a fixed blend ratio regardless of the observation condition.
As shown in Table~\ref{tab:omega_gate} at resolution 32, the $\bar{g}$ values are larger compared to resolutions 64 and 128.
Since $\bar{g}$ represents the degree of memory usage via the adaptive memory gate, this confirms that the gate adjusts to increase memory usage at low resolution.
For KS, $\bar{g} \approx 0.70$, and for Burgers', $\bar{g} = 0.0267$, which is larger than the 0.0165 observed at resolution 64.

\paragraph{Medium resolution ($f=64$).}
AMGFNO outperforms all baselines across all viscosities at resolution 64.
Compared to the previous best, S4FFNO, nRMSE on KS decreases from 0.036 to 0.021 at $\nu=0.075$, from 0.011 to 0.007 at $\nu=0.1$, and from 0.004 to 0.003 at $\nu=0.125$.
On Burgers', nRMSE decreases from 0.037 to 0.018 compared to S4FFNO.
At resolution 64, the adaptive memory gate $\bar{g}$ varies substantially with viscosity, adjusting memory usage appropriately.
At $\nu=0.075$, $\bar{g} = 0.4720$, meaning approximately 47\% of the memory path and 53\% of the Markovian path are utilized as shown in Figure~\ref{fig:overview}.
At $\nu=0.125$, $\bar{g} = 0.0721$, meaning the memory path contributes only about 7\% while the Markovian path contributes about 93\%; that is, at $\nu=0.125$ the adaptive memory gate partially closes with a low $\bar{g}$ value.
This is consistent with the finding in Section~\ref{subsec:fixedalpha} that optimal memory weight depends on viscosity: Table~\ref{tab:omega_gate} shows that the adaptive memory gate $\bar{g}$ modulates accordingly across different $\nu$ values.

\paragraph{High resolution ($f=128$).}
At high resolution, the adaptive memory gate value is nearly closed---0.003 at $\nu=0.075$, 0.0008 at $\nu=0.1$, and 0.0002 at $\nu=0.125$---effectively guiding AMGFNO into the Markovian condition.
S4FFNO uses fixed memory and at resolution 128 actually performs slightly worse than FFNO: at $\nu=0.075$, S4FFNO's nRMSE increases from 0.006 to 0.008 compared to FFNO, indicating that memory usage is harmful at high resolution.
AMGFNO addresses this limitation by incorporating an adaptive memory gate to regulate memory usage, achieving performance comparable to or better than FFNO across all viscosities on KS.
On Burgers', AMGFNO achieves the lowest nRMSE of 0.020, compared to 0.099 for FFNO, 0.028 for Multi Input FFNO, and 0.030 for S4FFNO.

\begin{figure}[t]
    \centering
    \begin{subfigure}[t]{0.49\columnwidth}
        \centering
        \includegraphics[width=\linewidth]{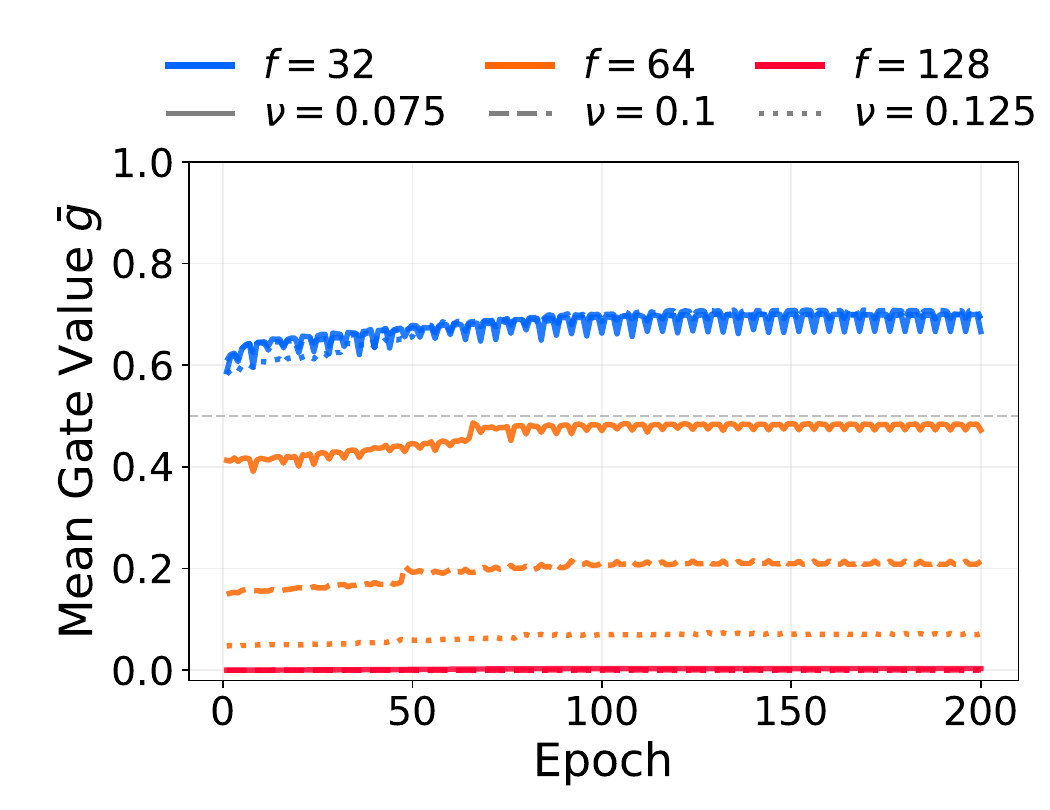}
        \caption{KS}
    \end{subfigure}\hfill
    \begin{subfigure}[t]{0.49\columnwidth}
        \centering
        \includegraphics[width=\linewidth]{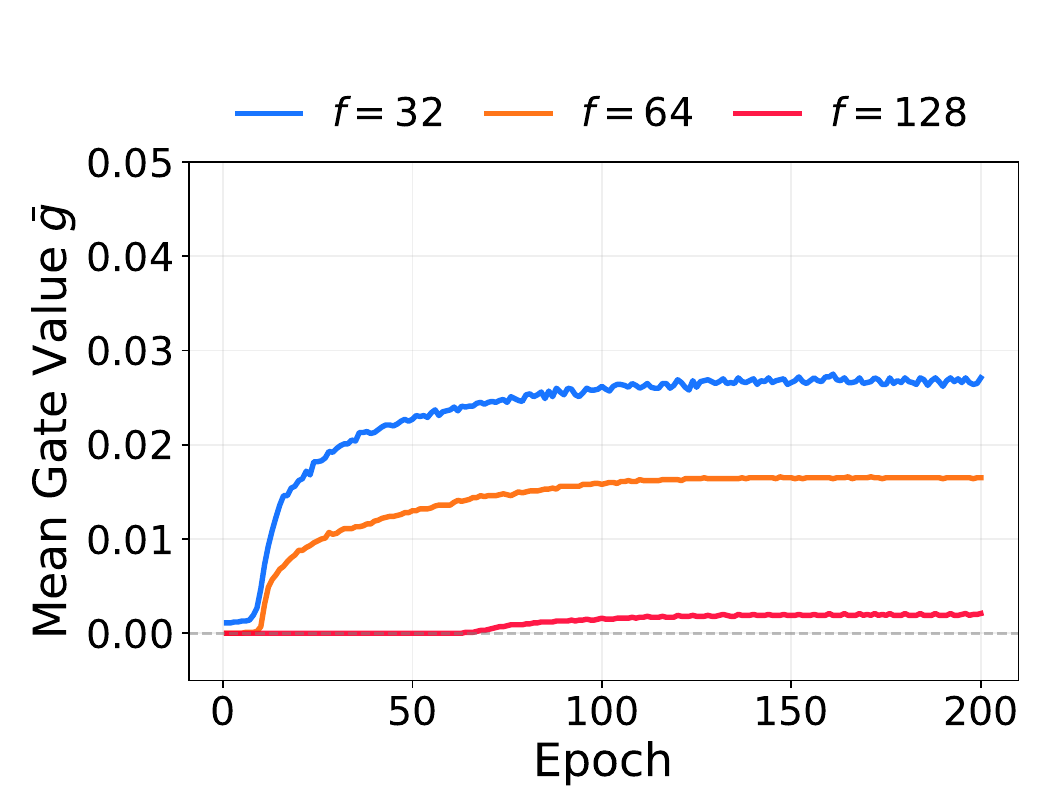}
        \caption{Burgers ($\nu=0.001$)}
    \end{subfigure}
    \caption{Adaptive memory gate behavior. Mean gate value $\bar{g}$ during training across resolutions $f$.}
    \label{fig:gate_behavior}
\end{figure}

\subsection{Additional Analysis}
\paragraph{Resolution and viscosity-dependent adaptation.}
\label{sebsec:additional}
As shown in Table~\ref{tab:omega_gate}, the learned gate values adapt to both resolution and viscosity.
For KS, the gate actively leverages memory at $f=32$ with $\bar{g} \approx 0.70$, then nearly closes at $f=128$ with $\bar{g} \leq 0.003$.
At medium resolution ($f=64$), viscosity-dependent variation is pronounced: $\bar{g}=0.472$ at $\nu=0.075$ versus $\bar{g}=0.072$ at $\nu=0.125$, demonstrating that the gate selectively opens in low-viscosity settings where high-frequency components remain significant and closes when the solution is well captured at the observation resolution.
For Burgers', gate values are relatively small across all resolutions ($\bar{g} \leq 0.027$), reflecting the PDE's intrinsic characteristic where the viscous term $\nu u_{xx}$ acts as a low-pass filter that monotonically attenuates high-frequency components.

\paragraph{Gate evolution during training.}
Figure~\ref{fig:gate_behavior} shows gate value evolution over epochs.
For the KS equation (left panel), the gate values are initialized differently by resolution and viscosity, showing distinct patterns from epoch 0.
At low resolution ($f=32$), the gates start large ($\bar{g} \approx 0.60$--$0.70$) across all viscosities.
At medium resolution ($f=64$), viscosity-dependent initialization emerges: $\bar{g} \approx 0.47$ for $\nu=0.075$, $\bar{g} \approx 0.21$ for $\nu=0.1$, and $\bar{g} \approx 0.07$ for $\nu=0.125$.
At high resolution ($f=128$), the gates are initialized near zero regardless of viscosity.
During training, these initial values are fine-tuned incrementally. 

The right panel shows gate evolution for Burgers' equation at $\nu=0.001$.
At epoch 0, all gate values start near zero regardless of resolution.
As training progresses, the gate values separate by resolution: the low-resolution gate increases substantially, while the high-resolution gate remains near zero.
This confirms that gate values are automatically learned without manual tuning per condition. Appendix~\ref{app:gate_loss} shows the co-evolution of gate values and training loss across all conditions.

\paragraph{Role of the frequency-aware gate.}
In our implementation, we optionally include a frequency-aware gate $\sigma(w_1 \log \omega_f + b_\omega)$ in addition to the content gate, which provides a conditional prior at initialization.
While the content gate alone can learn adaptive memory modulation across most settings, an ablation study (Appendix~\ref{app:ablation-scale-gate}) shows that the frequency-aware gate is most beneficial at high resolution in KS, where $\omega_f \to 0$ and it correctly suppresses unnecessary memory usage.

\begin{figure}[t]
    \centering
    \includegraphics[width=0.75\columnwidth]{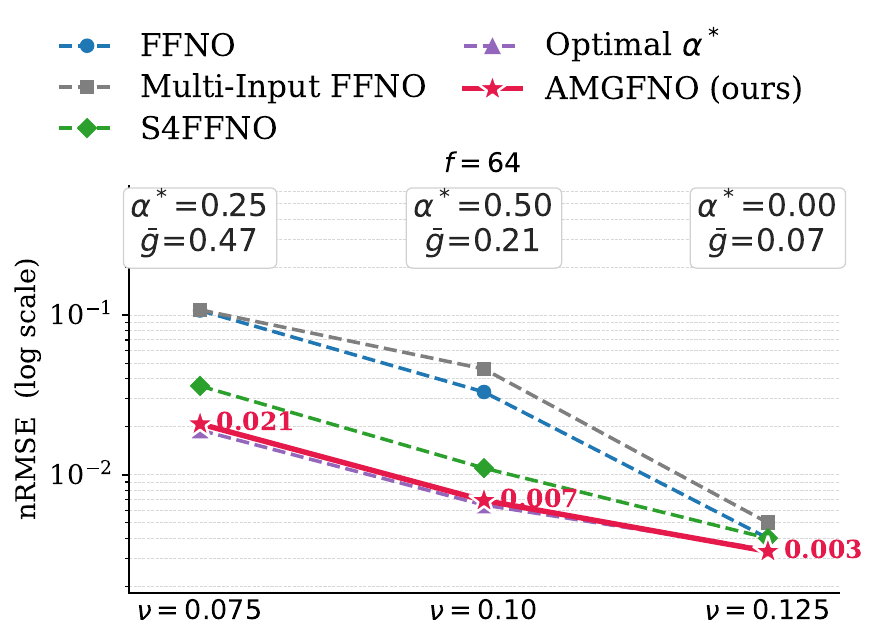}
    \caption{Performance comparison on the Kuramoto-Sivashinsky (KS) equation at resolution 64. FFNO, Multi-Input FFNO, S4FFNO baselines and optimal memory weight ($\alpha^*$) are compared to the proposed AMGFNO across different viscosities ($\nu \in \{0.075, 0.1, 0.125\}$). $\bar{g}$ denotes the adaptive memory gate value learned by AMGFNO.}
    \label{fig:compare_64}
    \vspace{-1em}
\end{figure}

\paragraph{Comparison with Optimal Fixed Scaling}
Figure~\ref{fig:compare_64} demonstrates the necessity of adaptive gating by comparing against the optimal fixed memory weight $\alpha^*$ from Section~\ref{subsec:fixedalpha} at resolution $f=64$ on the Kuramoto-Sivashinsky (KS) equation.
The optimal $\alpha^*$ varies across KS viscosities: $\alpha^*=0.25$ at $\nu=0.075$, $\alpha^*=0.50$ at $\nu=0.10$, and $\alpha^*=0.00$ at $\nu=0.125$, confirming that no single fixed value works across conditions.
AMGFNO's learned gate values $\bar{g}=0.47$, $0.21$, and $0.07$ also adapt across these viscosity conditions, achieving nRMSE of 0.021, 0.007, and 0.003 that closely matches the performance of the manually-tuned optimal at each condition.
Crucially, while finding $\alpha^*$ requires exhaustive search per condition, AMGFNO discovers the appropriate memory weight automatically through training on the KS dataset.
Results for resolutions 32 and 128 are in Appendix~\ref{app:compare_all_resolutions}.

\section{Related Work}
\paragraph{Neural operators for PDEs.}
The Fourier Neural Operator (FNO)~\cite{li2021fourier} and its factorized 
variant FFNO~\cite{FFNO} have become standard Markovian baselines for 
PDE operator learning. Other Markovian operators explore alternative 
inductive biases, including multi-scale spatial 
structure~\cite{unet_pde}, kernel-based attention~\cite{cao2021choose}, 
and transformer-style architectures with factorized or axial 
attention~\cite{li2022transformer, factformer}. Despite their 
architectural diversity, these models share the Markovian assumption, 
predicting the next state from the current one alone — a design that 
becomes inadequate when the current state does not fully resolve the 
underlying dynamics.
\paragraph{Memory in PDE models.}
Memory-augmented neural operators incorporate past states as additional 
input to relax the Markovian assumption. The earliest such variant, Multi-Input 
FFNO, simply concatenates a short fixed-length 
window of past states as additional input channels, providing only 
short-range, uncompressed memory. MemNO~\cite{S4FFNO} extends this idea 
by replacing the input window with an S4 layer~\cite{gu2022efficiently} 
that compresses the entire past trajectory along the time axis, and 
provides theoretical and empirical evidence that such long-range memory 
is beneficial under partial observation. Concretely, MemNO inserts a 
single S4 layer between Markovian FFNO layers and adds the resulting 
temporal-memory representation back via a residual connection. In both 
designs, however, the memory weight is fixed regardless of the 
input — yet, as we show, the optimal memory weight depends on 
observation conditions.
\paragraph{Gating mechanisms.}
The variation in optimal memory weight points to a need for 
input-dependent control of memory usage.
Such control is well established in sequence modeling, where 
multiplicative gates have a long history — from 
LSTM~\cite{hochreiter1997lstm} and GRU~\cite{cho2014gru} to Highway 
Networks~\cite{srivastava2015highway} and selective state-space models 
such as Mamba~\cite{gu2023mamba}; in all cases, gates regulate 
information flow as a function of the input. In neural operators, by 
contrast, memory has been integrated by a fixed, input-independent 
addition~\cite{S4FFNO}.
\paragraph{Position of our work.}
AMGFNO bridges these two lines of work. We inherit the temporal-memory 
architecture of MemNO but replace its fixed fusion with an adaptive 
gate that combines a content gate and frequency-aware gate. Unlike prior memory-augmented neural operators, AMGFNO modulates modulates memory weight through a learnable gate. 
\section{Conclusion and Future Work}
We demonstrate the necessity of adaptive memory control in neural operators for time-dependent PDEs.
Preliminary studies reveal that optimal memory weight varies across observation conditions.
AMGFNO incorporates an adaptive memory gate that learns to modulate memory usage based on the current condition.
At low resolution where unobserved high-frequency modes are significant, AMGFNO achieves up to 2--5$\times$ error reduction compared to S4FFNO.
At high resolution, where most spectral information is observable, the gate reduces memory usage, matching or exceeding purely Markovian models.
The gate adapts without manual tuning, with learned values adjusting according to resolution and spectral characteristics through training.

Future work includes extending AMGFNO to higher-dimensional PDEs, exploring alternative condition descriptors beyond spectral energy ratios, and evaluating robustness under observational noise and distribution shift.

\section*{Acknowlegments}
This work was partly supported by the Institute for Information \& Communications Technology Planning \& Evaluation (IITP) grants funded by the Korean government (MSIT) 
(No. RS-2026-25526850, High-Efficiency Neural Networks for Artificial General Intelligence (HERMES-Net)), 
Samsung Electronics Co., Ltd. (No. G01240136, KAIST Semiconductor Research Fund (2nd))
and the Korea Advanced Institute of Science and Technology (KAIST) grant funded by the Korea government (MSIT) (No. G04240001, Physics-inspired Deep Learning).
J.~Choi was additionally supported by the KAIST Jang Young Sil Fellow Program.

\bibliography{references}
\bibliographystyle{icml2026}

\newpage
\appendix
\onecolumn
\section{Preliminaries Study: Memory Ablation under Additive Fusion}
\label{app:additive_motivation}
\begin{figure*}[h!]
    \centering
    \begin{subfigure}[t]{0.245\textwidth}
        \centering
        \includegraphics[width=\linewidth]{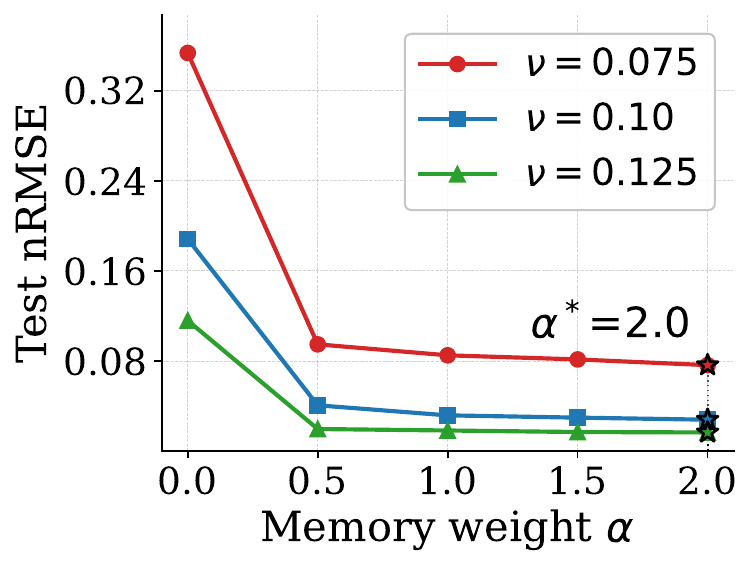}
        \caption{$f = 32$}
    \end{subfigure}\hfill
    \begin{subfigure}[t]{0.245\textwidth}
        \centering
        \includegraphics[width=\linewidth]{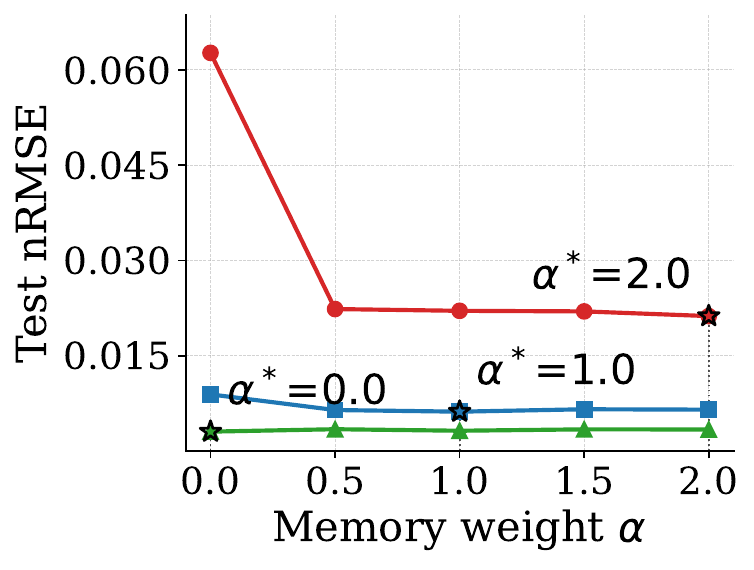}
        \caption{$f = 64$}
    \end{subfigure}\hfill
    \begin{subfigure}[t]{0.245\textwidth}
        \centering
        \includegraphics[width=\linewidth]{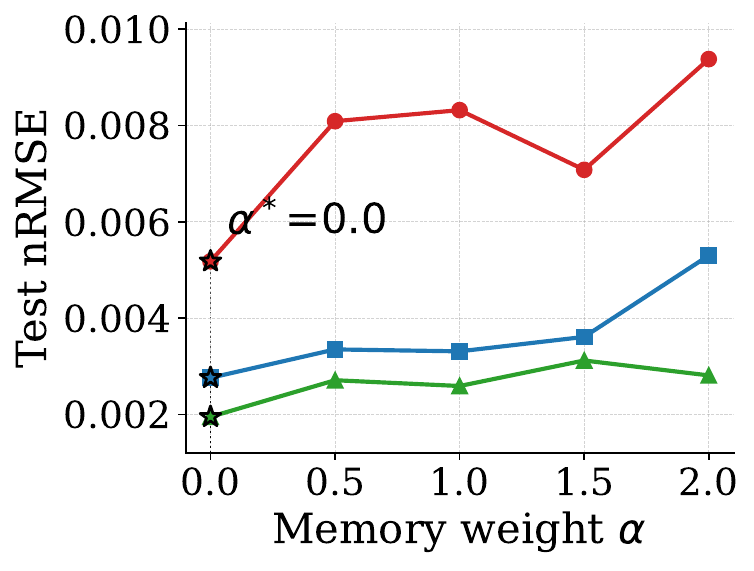}
        \caption{$f = 128$}
    \end{subfigure}\hfill
    \begin{subfigure}[t]{0.225\textwidth}
        \centering
        \includegraphics[width=\linewidth]{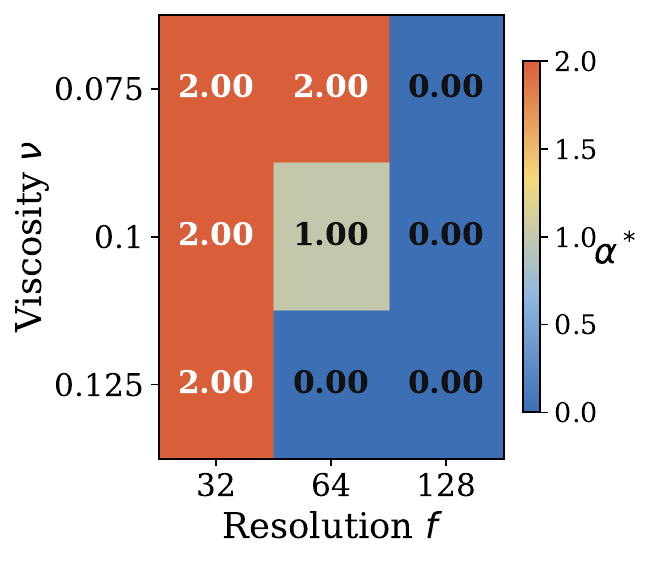}
        \caption{Optimal $\alpha^*$}
    \end{subfigure}
    \caption{Effect of memory weight $\alpha$ on nRMSE across resolutions and viscosity values on the KS equation. Each panel shows nRMSE as a function of $\alpha$ for three viscosity values ($\nu \in \{0.075, 0.10, 0.125\}$). The rightmost heatmap summarizes optimal memory weight $\alpha^*$ per setting.}
    \label{fig:motivation_additive}
\end{figure*}

The motivation analysis in Section~\ref{subsec:fixedalpha} parameterizes the memory weight via a convex blend, which is form-consistent with the adaptive memory gate of our main method. For completeness, we also report the same ablation under an additive parameterization that more directly mirrors the original residual fusion of S4FFNO~\cite{S4FFNO}.

Using the same notation as in Section~\ref{subsec:fixedalpha}, let $\mathbf{h}_t \in \mathbb{R}^{N \times H}$ be the hidden state at the output of the second FFNO layer and  $\mathbf{z}_t = \mathcal{M}(\mathbf{h}_{1:t}^{(\ell^{*}-1)})$ the output of the S4 memory branch. We replace the convex blend with the additive update
\begin{equation}
    \mathbf{h}_{t+1} \leftarrow \alpha \cdot \mathbf{z}_t + \mathbf{h}_t,
\end{equation}
which preserves the residual path in full and scales the memory weight by $\alpha$. Under this parameterization, S4FFNO is recovered exactly at $\alpha=1$, and $\alpha=0$ again corresponds to a pure Markovian model that bypasses the memory branch. Allowing $\alpha > 1$ amplifies memory beyond its default contribution.

We sweep $\alpha \in \{0.0, 0.5, 1.0, 1.5, 2.0\}$ on the Kuramoto--Sivashinsky equation under the same setting as Section~\ref{subsec:fixedalpha} ($f \in \{32, 64, 128\}$, $\nu \in \{0.075, 0.1, 0.125\}$).

\paragraph{Observation 1: More memory is beneficial at low resolution.}
At $f=32$, $\alpha^*=2.0$ across all three viscosity values, and larger $\alpha$ consistently reduces nRMSE.

\paragraph{Observation 2: Memory isn't useful at high resolution.}
At $f=128$, $\alpha^*=0.0$ for all viscosity values, and larger $\alpha$ monotonically degrades performance.

\paragraph{Observation 3: The optimal memory weight $\alpha^*$ depends jointly on resolution and viscosity.}
At $f=64$, $\alpha^*$ varies across viscosity values: $\alpha^*=2.0$ at $\nu=0.075$, $\alpha^*=1.0$ at $\nu=0.1$, and $\alpha^*=0.0$ at $\nu=0.125$. Higher viscosity reduces the high-frequency content of the PDE solution, diminishing the practical benefit of memory.

\paragraph{Implications.}
The same three qualitative conclusions emerge under both the convex (Section~\ref{subsec:fixedalpha}) and additive parameterizations: memory helps most at low resolution, hurts at high resolution, and its optimal weight depends jointly on resolution and viscosity. The robustness of these observations to the choice of fusion form supports our main motivation that no single fixed memory weight is appropriate across operating conditions, and that the gate must be learned and conditioned on the high-frequency structure of the input.

\section{Choice of $\omega_f$ Computation: Fixed vs.\ Per-Batch}
\label{app:omega_ablation}
We compare two strategies for obtaining $\omega_f$: (i) the proposed
\emph{fixed} variant, which estimates $\omega_f$ once as the training-set
mean and holds it constant; and (ii) a \emph{per-batch} variant that
recomputes $\omega_f$ from each training and inference batch on the fly.
Both variants share the same gate architecture, initialization, and
training schedule.
Table~\ref{tab:omega_ablation} reports final test nRMSE on the KS equation across all nine $(f, \nu)$ settings. The fixed variant consistently outperforms per-batch recomputation, with performance degradation ranging from $5.8\%$ to $20.1\%$ when $\omega_f$ is recomputed per batch. This supports our design choice in Section~\ref{subsec:gate} to use a fixed, training-set-averaged $\omega_f$ rather than dynamic per-batch estimation.

We hypothesize that per-batch $\omega_f$ introduces noisy, sample-dependent signals that destabilize gate learning. In contrast, the fixed variant provides a stable global measure of spectral complexity tied to the resolution--viscosity condition, allowing the gate to learn consistent memory modulation patterns. The fixed approach also reduces computational overhead by eliminating redundant FFT operations at each forward pass.

\begin{table}[t]
\centering
\small
\caption{Final test nRMSE on KS for fixed vs.\ per-batch $\omega_f$. The fixed variant outperforms per-batch recomputation in all 9 $(f, \nu)$ settings, with degradation ranging from $5.8\%$ to $20.1\%$, supporting our design choice in Section~3.3.}
\label{tab:omega_ablation}
\begin{tabular}{cc|cc|c}
\toprule
$f$ & $\nu$ & fixed & per-batch & $\Delta\%$ \\
\midrule
32 & 0.075 & 0.0616 & 0.0651 & $+5.8$ \\
32 & 0.10 & 0.0226 & 0.0249 & $+10.1$ \\
32 & 0.125 & 0.0142 & 0.0168 & $+18.2$ \\
64 & 0.075 & 0.0208 & 0.0250 & $+20.1$ \\
64 & 0.10 & 0.0069 & 0.0077 & $+11.9$ \\
64 & 0.125 & 0.0033 & 0.0039 & $+17.2$ \\
128 & 0.075 & 0.0058 & 0.0068 & $+18.5$ \\
128 & 0.10 & 0.0030 & 0.0033 & $+11.1$ \\
128 & 0.125 & 0.0019 & 0.0020 & $+5.8$ \\
\bottomrule
\end{tabular}
\end{table}

\section{Ablation: Effect of the Frequency-Aware Gate}
\label{app:ablation-scale-gate}

\begin{table*}[t]
\centering
\caption{Ablation of the frequency-aware gate in AMGFNO. Bold indicates best performance per setting.}
\label{tab:ablation}
\small
\setlength{\tabcolsep}{3.5pt}
\begin{tabular}{l ccc ccc ccc ccc}
\toprule
& \multicolumn{9}{c}{Kuramoto--Sivashinsky} & \multicolumn{3}{c}{Burgers}\\
\cmidrule(lr){2-10}\cmidrule(lr){11-13}
                              & \multicolumn{3}{c}{$f=32$}                            & \multicolumn{3}{c}{$f=64$}                            & \multicolumn{3}{c}{$f=128$}                          &                  &                  &                  \\
\cmidrule(lr){2-4}\cmidrule(lr){5-7}\cmidrule(lr){8-10}
Method                        & $\nu{=}.075$    & $\nu{=}.10$     & $\nu{=}.125$    & $\nu{=}.075$    & $\nu{=}.10$     & $\nu{=}.125$    & $\nu{=}.075$    & $\nu{=}.10$     & $\nu{=}.125$    & $f{=}32$         & $f{=}64$         & $f{=}128$        \\
\midrule
AMGFNO                        & 0.0616          & \textbf{0.0226} & \textbf{0.0142} & 0.0208          & 0.0069          & 0.0033          & \textbf{0.0058} & \textbf{0.0030} & \textbf{0.0019} & \textbf{0.0244} & 0.0179          & 0.0202          \\
\quad w/o frequency-aware          & \textbf{0.0573} & 0.0231          & 0.0151          & \textbf{0.0185} & \textbf{0.0058} & \textbf{0.0033} & 0.0066          & 0.0041          & 0.0027          & 0.0253          & \textbf{0.0140} & \textbf{0.0122} \\
\bottomrule
\end{tabular}
\end{table*}

Table~\ref{tab:ablation} reports the effect of removing the frequency-aware gate from AMGFNO, retaining only the content gate $\mathbf{g}^{(c)}_t$. The frequency-aware gate $\sigma(\omega_f \cdot w_1 + b_\omega)$ provides a spectral prior that modulates the baseline gate budget based on the fraction of unobserved high-frequency modes. We evaluate whether this prior offers a consistent advantage or whether the content gate alone can learn appropriate memory modulation from data.

\paragraph{KS: Frequency-aware gate helps at high resolution.}
On KS, $\omega_f$ varies dramatically across resolutions: $\omega_{32}=0.584$, $\omega_{64}=0.083$, $\omega_{128}=0.0003$. At high resolution ($f=128$), AMGFNO consistently outperforms w/o frequency-aware across all three viscosity settings. This aligns with Theorem~\ref{thm:main-amg-highres}: when $\omega_f \to 0$ (minimal high-frequency loss), the frequency-aware gate correctly drives the gate budget toward zero, suppressing unnecessary memory usage and avoiding the complexity cost of maintaining memory when it provides no benefit.

At low and mid resolution ($f \in \{32, 64\}$), results are mixed, with w/o frequency-aware winning in 4 out of 6 settings. However, AMGFNO still outperforms at the most challenging low-viscosity setting ($\nu=0.075$, $f=32$), where unobserved high-frequency modes are most prevalent and memory is essential.

\paragraph{Burgers: Frequency-aware gate has limited impact.}
On Burgers', $\omega_f$ is uniformly small across all resolutions: $\omega_{32}=0.0062$, $\omega_{64}=0.0012$, $\omega_{128}=0.0002$. Here, w/o frequency-aware matches or outperforms AMGFNO at mid and high resolutions ($f \in \{64, 128\}$). When $\omega_f$ provides little differentiation across conditions, the frequency-aware gate offers minimal additional signal beyond what the content gate learns from data.

\paragraph{Implications.}
The frequency-aware gate is most valuable when $\omega_f$ varies meaningfully across observation conditions, as demonstrated at high resolution on KS where it consistently improves performance. When $\omega_f$ is uniformly small (as in Burgers'), the content gate alone suffices. This suggests a direction for future work: PDE-adaptive initialization of $w_1$ and $b_\omega$ based on the dataset's $\omega_f$ distribution could improve robustness across diverse PDE families. For instance, setting $w_1 \propto \max_f \omega_f - \min_f \omega_f$ would amplify the frequency-aware gate's influence when spectral variation is high and suppress it when variation is negligible.

\label{app:compare_all_resolutions}
\begin{figure}[t]
    \centering
    \begin{subfigure}[t]{0.32\textwidth}
        \centering
        \includegraphics[width=\textwidth, height=7cm, keepaspectratio]{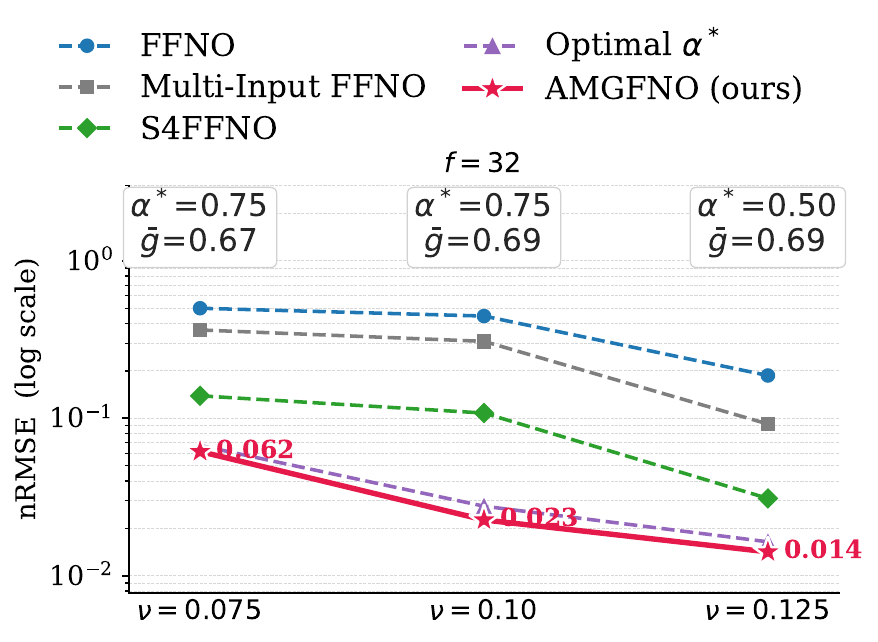}
        \caption{KS $f=32$}
        \label{fig:ks_n32}
    \end{subfigure}
    \hfill
    \begin{subfigure}[t]{0.32\textwidth}
        \centering
        \includegraphics[width=\textwidth, height=7cm, keepaspectratio]{fig/compare_64.pdf}
        \caption{KS $f=64$}
        \label{fig:ks_n64}
    \end{subfigure}
    \hfill
    \begin{subfigure}[t]{0.32\textwidth}
        \centering
        \includegraphics[width=\textwidth, height=7cm, keepaspectratio]{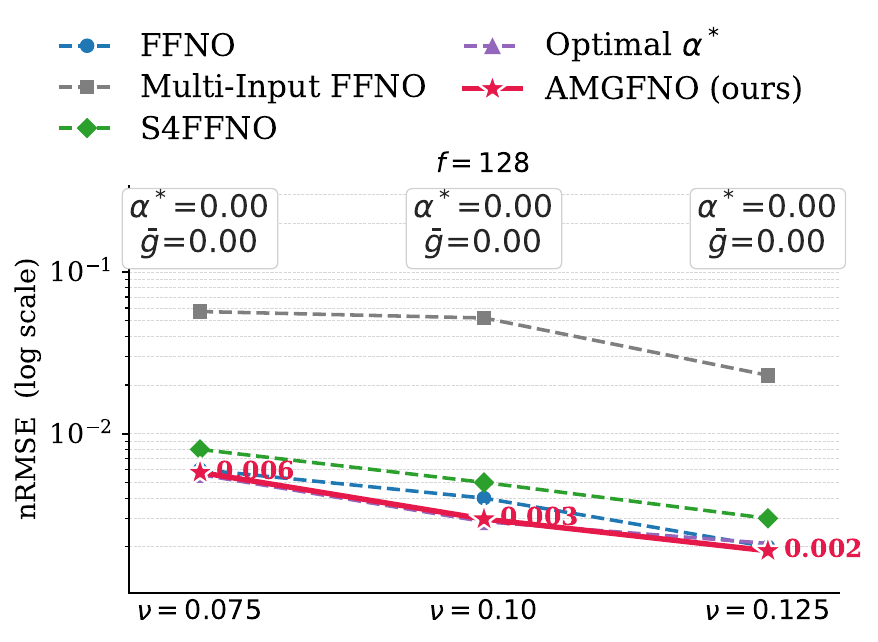}
        \caption{KS $f=128$}
        \label{fig:ks_n128}
    \end{subfigure}
    \caption{Performance comparison on KS across resolutions $f \in \{32, 64, 128\}$ and viscosities $\nu \in \{0.075, 0.10, 0.125\}$. Optimal $\alpha^*$ denotes a model trained with the manually tuned memory weight $\alpha^*$ that achieves best performance per setting.}
    \label{fig:compare_all}
\end{figure}

\section{Performance Comparison Across Resolutions}
Figure~\ref{fig:compare_all} extends Section~\ref{sebsec:additional} across all three resolutions. Optimal $\alpha^*$ represents an oracle baseline where the memory weight is manually tuned for each $(f, \nu)$ condition, providing an upper bound on what can be achieved with a fixed memory weight.
AMGFNO achieves performance comparable to Optimal $\alpha^*$ across all conditions by learning adaptive memory gate values through training. At $f=32$, where substantial high-frequency content is unobserved, AMGFNO learns gate values $\bar{g} = 0.67$--$0.69$ and achieves nRMSE of 0.062, 0.023, and 0.014 across the three viscosities, closely matching the oracle performance. At $f=64$, the learned gate values adapt across viscosities ($\bar{g} = 0.47, 0.21, 0.07$), achieving nRMSE of 0.021, 0.007, and 0.003. At $f=128$, where minimal high-frequency information is lost, AMGFNO learns $\bar{g} \approx 0.00$ and achieves nRMSE of 0.006, 0.003, and 0.002.

These results demonstrate that AMGFNO automatically learns condition-appropriate memory usage without requiring the exhaustive hyperparameter search needed to obtain Optimal $\alpha^*$. The learned gate behavior is consistent with the Mori-Zwanzig principle: opening when unobserved high-frequency modes are prevalent and closing when the Markovian assumption suffices.

\section{Relationship Between Gate Value and Training Loss}
\label{app:gate_loss}

\begin{figure}[h!]
    \centering
    \begin{subfigure}[t]{\textwidth}
        \centering
        \begin{subfigure}[t]{0.32\textwidth}
            \centering
            \includegraphics[width=\textwidth]{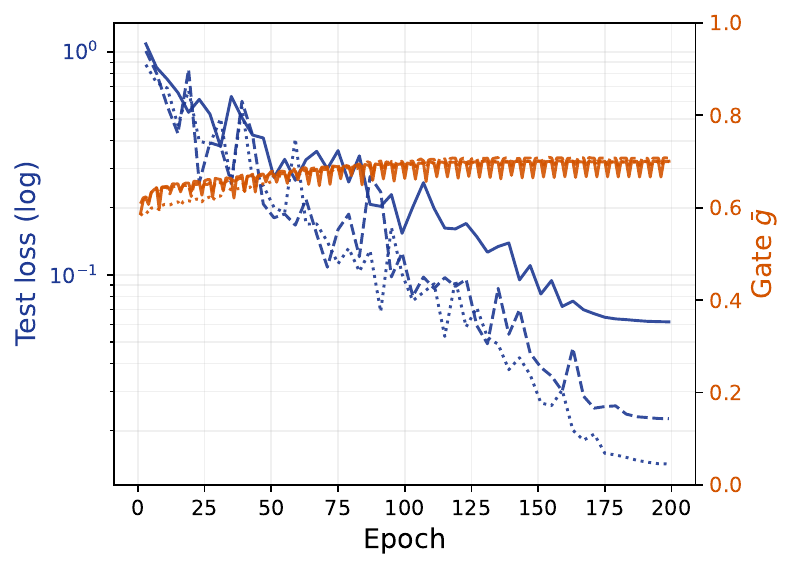}
            \caption{$f=32$}
        \end{subfigure}
        \hfill
        \begin{subfigure}[t]{0.32\textwidth}
            \centering
            \includegraphics[width=\textwidth]{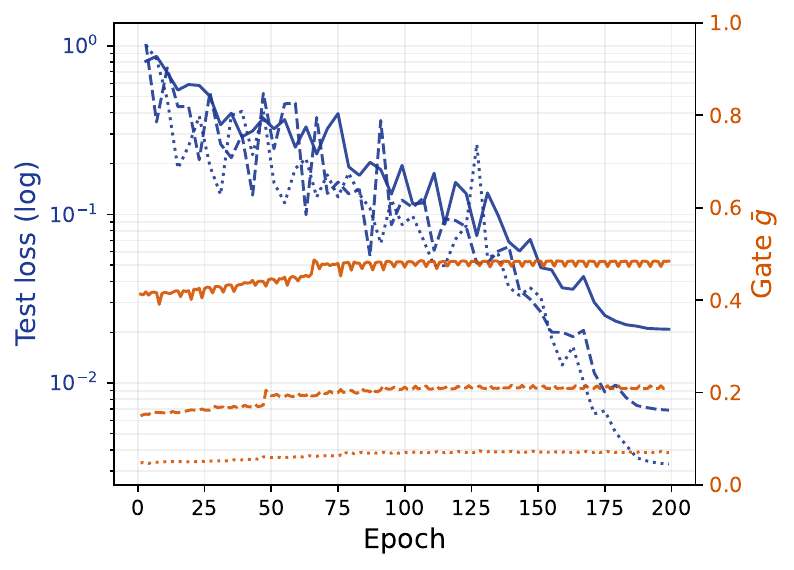}
            \caption{$f=64$}
        \end{subfigure}
        \hfill
        \begin{subfigure}[t]{0.32\textwidth}
            \centering
            \includegraphics[width=\textwidth]{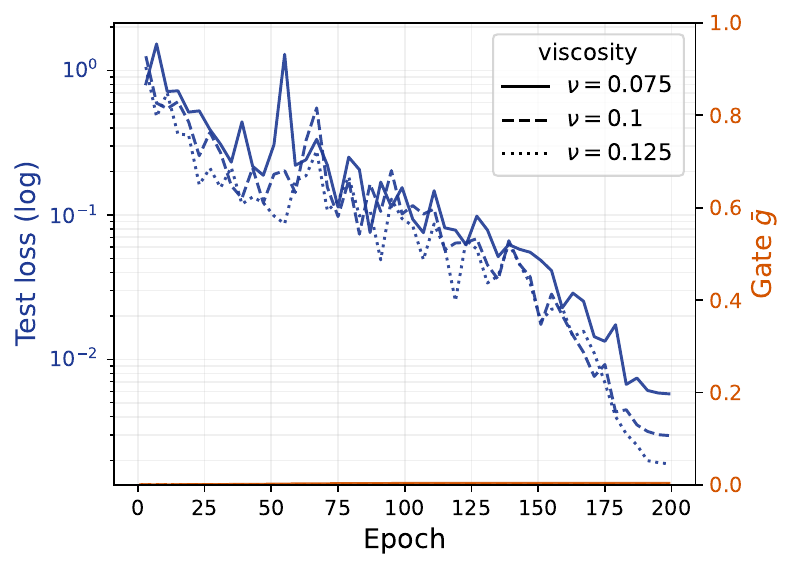}
            \caption{$f=128$}
        \end{subfigure}
        \label{fig:gate_loss_ks}
    \end{subfigure}

    \vspace{1em}

    \begin{subfigure}[t]{\textwidth}
        \centering
        \begin{subfigure}[t]{0.32\textwidth}
            \centering
            \includegraphics[width=\textwidth]{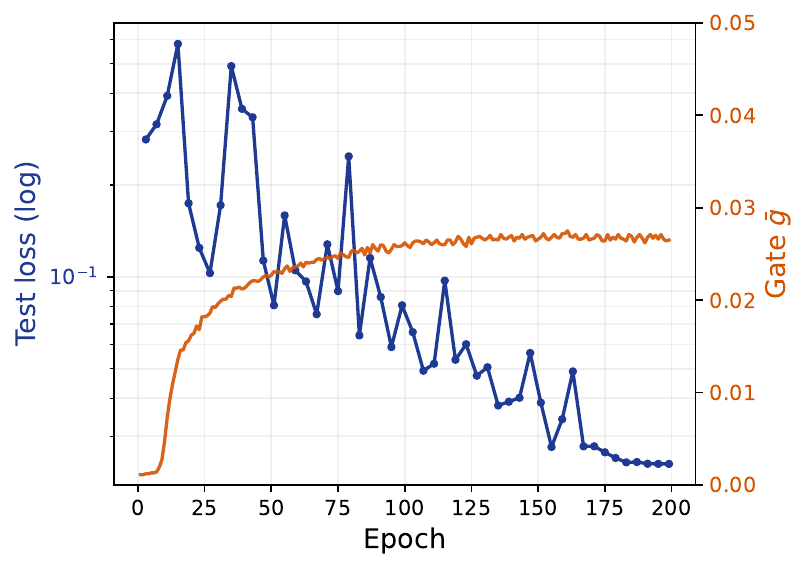}
            \caption{$f=32$}
        \end{subfigure}
        \hfill
        \begin{subfigure}[t]{0.32\textwidth}
            \centering
            \includegraphics[width=\textwidth]{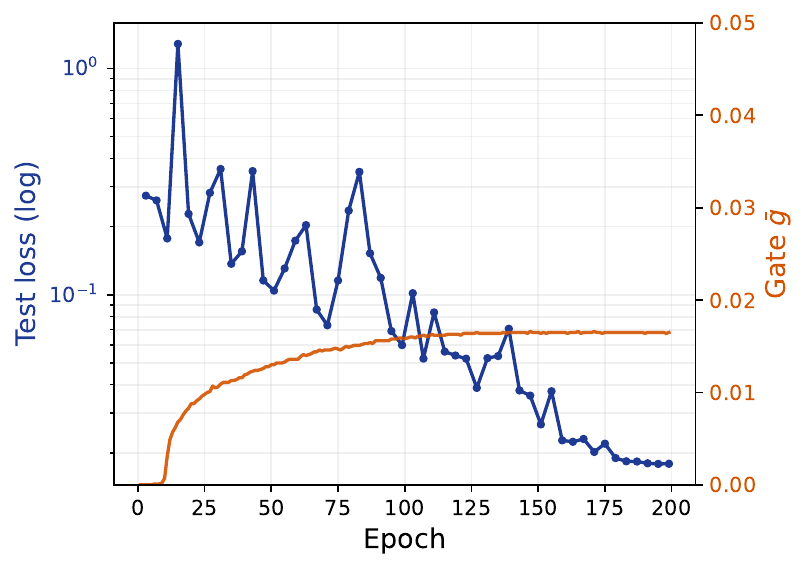}
            \caption{$f=64$}
        \end{subfigure}
        \hfill
        \begin{subfigure}[t]{0.32\textwidth}
            \centering
            \includegraphics[width=\textwidth]{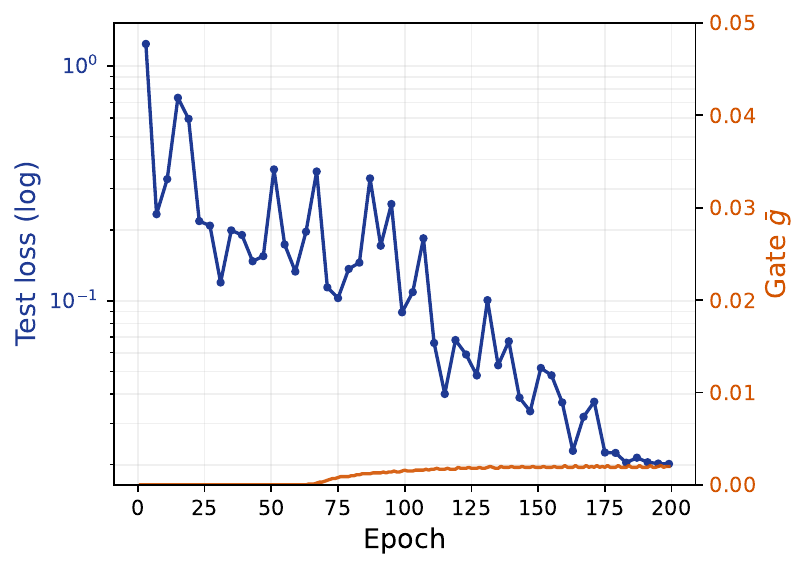}
            \caption{$f=128$}
        \end{subfigure}
        \label{fig:gate_loss_burgers}
    \end{subfigure}

    \caption{Evolution of test loss (blue, left axis, log scale) and adaptive memory gate value
    $\bar{g}$ (orange, right axis) over training epochs for AMGFNO.
    \textit{(a, b, c)} KS equation at resolutions $f \in \{32, 64, 128\}$ for viscosities
    $\nu \in \{0.075, 0.10, 0.125\}$ (solid, dashed, and dotted lines, respectively).
    \textit{(d, e, f)} Burgers' equation at resolutions $f \in \{32, 64, 128\}$ with $\nu = 0.001$.}
    \label{fig:gate_loss}
\end{figure}

Figure~\ref{fig:gate_loss} shows the co-evolution of test loss and adaptive memory gate value $\bar{g}$ over 200 training epochs. Test loss decreases monotonically while the gate converges to a stable level, with gate convergence typically preceding final loss convergence.
The converged gate values reflect the need for memory. For KS, $\bar{g}$ stratifies by resolution and viscosity: at $f=32$, $\bar{g} \approx 0.67$--$0.70$ across all viscosities; at $f=64$, $\bar{g}$ ranges from $0.47$ (low viscosity) to $0.07$ (high viscosity); at $f=128$, $\bar{g} \approx 0$ for all conditions. For Burgers', the gate converges to small values ($\bar{g} \approx 0.027$ at $f=32$, decreasing to $\approx 0.001$ at $f=128$), consistent with uniformly small $\omega_f$ values. These dynamics confirm that the adaptive memory gate discovers appropriate memory usage purely through training.

\section{Experimental Details}
\label{app:exp-details}
\label{app:data}

This appendix provides the experimental details omitted from
Section~\ref{sec:setup}, including PDE equations, dataset construction,
resolution processing, baseline architectures, training hyperparameters,
evaluation protocol, and the computation of the spectral ratio used by AMGFNO.
Unless otherwise stated, our experimental setup follows S4FFNO~\citep{S4FFNO}.
In particular, the Kuramoto--Sivashinsky datasets follow the low-viscosity
benchmark generated in S4FFNO, which is based on the PDE-Refiner/LPSDA
generation pipeline~\citep{lippe2023pde,brandstetter2022lie}. The Burgers'
dataset is taken from the public PDEBench benchmark~\citep{takamoto2022pdebench}.

\subsection{PDE Datasets and Parameters}
\label{app:exp-pdes}

\paragraph{Kuramoto--Sivashinsky equation.}
The Kuramoto--Sivashinsky equation is given by
\begin{equation}
    u_t(t,x)
    +
    u(t,x)u_x(t,x)
    +
    u_{xx}(t,x)
    +
    \nu u_{xxxx}(t,x)
    =
    0,
    \qquad
    (t,x)\in[0,T]\times[0,L],
    \label{eq:app-ks}
\end{equation}
with periodic boundary conditions. Following S4FFNO, we use low-viscosity KS
datasets because low viscosity produces solutions with substantial
high-frequency Fourier components. We use
\begin{equation}
    \nu\in\{0.075,0.10,0.125\}.
\end{equation}
The spatial domain length is
\begin{equation}
    L=64,
\end{equation}
and the final rollout time used in our experiments is
\begin{equation}
    T=2.5.
\end{equation}
Each trajectory is observed at 26 equispaced time points, giving 25 prediction
timesteps with
\begin{equation}
    \Delta t=0.1.
\end{equation}
The high-resolution reference trajectories are generated at spatial resolution
512. We use 2048 training trajectories and 256 held-out evaluation trajectories,
following the S4FFNO setup.

The initial condition is a superposition of sinusoidal waves:
\begin{equation}
    u_0(x)
    =
    \sum_{i=0}^{20}
    A_i
    \sin
    \left(
        \frac{2\pi k_i}{L}x+\phi_i
    \right),
    \label{eq:app-ks-ic}
\end{equation}
where, for each trajectory,
\begin{equation}
    A_i\sim{\rm Unif}[-0.5,0.5],
    \qquad
    k_i\sim{\rm Unif}\{1,\ldots,8\},
    \qquad
    \phi_i\sim{\rm Unif}[0,2\pi].
\end{equation}
The data generation follows the method of lines~\citep{schiesser1991numerical}.
Spatial derivatives are computed pseudospectrally using Fast Fourier transforms
\citep{cooley1969finite}, and the resulting ODE system is solved with an
implicit Runge--Kutta method of the Radau IIA family of order 5
\citep{hairer1996solving}, implemented through the SciPy numerical stack
\citep{virtanen2020scipy}.

\paragraph{Burgers' equation.}
The viscous Burgers' equation is given by
\begin{equation}
    u_t(t,x)
    +
    u(t,x)u_x(t,x)
    =
    \nu u_{xx}(t,x),
    \qquad
    (t,x)\in[0,T]\times[0,L],
    \label{eq:app-burgers}
\end{equation}
with periodic boundary conditions. We use the public PDEBench Burgers' dataset
\citep{takamoto2022pdebench} with viscosity
\begin{equation}
    \nu=0.001.
\end{equation}
The highest available spatial resolution is 1024. Following S4FFNO, we use
2048 samples for training and reserve 10\% of the full dataset for testing. We
evaluate the first 20 prediction timesteps because after this period the
diffusion term attenuates most high-frequency components and the solution
changes slowly. The final time is
\begin{equation}
    T=1.4,
\end{equation}
and the rollout horizon contains 20 prediction timesteps.

\begin{table}[t]
\centering
\caption{Dataset and rollout summary.}
\label{tab:app-data-summary}
\small
\setlength{\tabcolsep}{5pt}
\begin{tabular}{lcccc}
\toprule
Dataset & $\nu$ & high-res. grid & train & rollout steps \\
\midrule
KS
& $\{0.075,0.10,0.125\}$
& 512
& 2048
& 25 \\
Burgers'
& $0.001$
& 1024
& 2048
& 20 \\
\bottomrule
\end{tabular}
\end{table}

\subsection{Resolution Construction}
\label{app:exp-resolution}

For both PDEs, we evaluate all models at
\begin{equation}
    f\in\{32,64,128\}.
\end{equation}
For a domain of length $L$, the equispaced grid at resolution $f$ is
\begin{equation}
    S_f
    :=
    \left\{
        x_r=\frac{rL}{f}
        :
        r=0,\ldots,f-1
    \right\}.
    \label{eq:app-grid}
\end{equation}
The high-resolution reference trajectories are first generated or loaded, and
lower-resolution observations are constructed from the highest-resolution
trajectory. For one-dimensional datasets, S4FFNO constructs lower-resolution
trajectories by cubic interpolation from the highest-resolution grid. We follow
the same procedure. Thus, for each resolution $f$, every model is trained and
tested on the corresponding grid $S_f$.

For trajectory $i$ and timestep $j$, the resolved state is
\begin{equation}
    u_j^{f,(i)}
    :=
    \mathcal I_f
    \left[
        u^{(i)}(t_j,\cdot)
    \right]
    \in\mathbb R^{f\times V},
    \label{eq:app-resolved-state}
\end{equation}
where $\mathcal I_f$ denotes cubic interpolation from the highest-resolution
reference grid to $S_f$, and $V$ is the number of physical channels. All models
use the same resolved states $u_j^{f,(i)}$ and the same grid encoding.

\paragraph{Relation to the theory.}
Appendix~\ref{app:theory} analyzes an ideal Fourier projection $P_f$ in order
to isolate the unresolved spectral-tail mechanism. The experiments follow the
S4FFNO data protocol and construct lower-resolution inputs by cubic
interpolation for comparability with prior work. The empirical quantity
$\omega_f$ is nevertheless computed from high-resolution training trajectories
using Fourier coefficients, as described in Appendix~\ref{app:exp-omega}.

\subsection{Computation of the Spectral Ratio}
\label{app:exp-omega}

By the Nyquist--Shannon sampling theorem~\citep{shannon1949communication}, an
observation at spatial resolution $f$ can resolve Fourier modes only up to
approximately $\lfloor f/2\rfloor$. For a full reference state $\phi$, define
the unresolved-mode ratio
\begin{equation}
    \omega_f(\phi)
    :=
    \frac{
        \sum_{|n|>\lfloor f/2\rfloor}|a_n(\phi)|^2
    }{
        \sum_{n\in\mathbb Z}|a_n(\phi)|^2
    },
    \label{eq:app-exp-omega-state}
\end{equation}
where $a_n(\phi)$ denotes the Fourier coefficient of $\phi$. For vector-valued
states, the squared magnitude is summed over physical channels.

In practice, the continuous Fourier coefficients are approximated using the
discrete Fourier transform of the highest-resolution trajectory available in
the dataset~\citep{cooley1969finite}. For KS, the highest resolution is 512; for
Burgers', it is 1024. For each condition $(f,\rho)$, we compute the empirical
training-set estimate
\begin{equation}
    \widehat\omega_{f,\rho}
    :=
    \frac{1}{N_{\rm train}T_{\rm obs}}
    \sum_{i=1}^{N_{\rm train}}
    \sum_{j=0}^{T_{\rm obs}-1}
    \omega_f
    \left(
        u_j^{{\rm high},(i)}
    \right),
    \label{eq:app-exp-omega-estimate}
\end{equation}
where $u_j^{{\rm high},(i)}$ is the high-resolution training reference state.
This provides a condition-level estimate of how much spectral energy is lost at
each observation resolution. It is computed only from training trajectories and
is held fixed during evaluation; no high-resolution test information is used.

The scalar frequency-aware gate in AMGFNO uses
\begin{equation}
    \sigma
    \left(
        w_1\log(\widehat\omega_{f,\rho}+\varepsilon_{\rm num})
        +
        b_\omega
    \right),
    \label{eq:app-exp-frequency-gate}
\end{equation}
where $\varepsilon_{\rm num}>0$ is a small numerical constant used only for
finite-precision stability. The asymptotic theory in Appendix~\ref{app:theory}
uses the corresponding zero-tail convention.

\paragraph{Note on observability.}
The condition-level estimate $\widehat{\omega}_{f,\rho}$ is computed from 
high-resolution training references and held fixed during training and 
evaluation. Thus, it is a calibration constant for the observation condition 
$(f,\rho)$ rather than a per-sample test-time feature, and no high-resolution 
test information is used. For deployment to new conditions without 
high-resolution training references, this scalar could be replaced by 
observable proxies, such as condition-level summaries of the resolved 
high-band spectrum, or by metadata-based gates conditioned on $f$ and available 
physical parameters $\rho$. We leave the systematic study of such proxies to 
future work.

\subsection{Compared Methods and Implementation Details}
\label{app:exp-baselines}

We compare AMGFNO against both Markovian and memory-augmented neural operators.
The baseline implementations and hyperparameters follow S4FFNO whenever
possible.

\paragraph{Shared input and output processing.}
All models use a spatial grid encoding. In one dimension, for a grid of
resolution $f$ on $[0,L]$, we use the normalized coordinate
\begin{equation}
    g_r:=\frac{x_r}{L},
    \qquad
    x_r=\frac{rL}{f}.
\end{equation}
The encoder maps the concatenation of the PDE state and grid coordinate to the
hidden dimension, and the decoder maps the final hidden representation back to
the PDE state.

\paragraph{FFNO.}
FFNO is based on the Factorized Fourier Neural Operator~\citep{FFNO}, which
builds on the original Fourier Neural Operator~\citep{li2021fourier}. The FFNO
layer combines a spectral convolution path with a pointwise feed-forward path.
The hidden dimension is 128 and the expanded hidden dimension in the FFNO MLP is
$4\cdot128$. We use all resolvable Fourier modes by setting
\begin{equation}
    k_{\max}
    =
    \left\lfloor \frac{f}{2}\right\rfloor.
\end{equation}

\paragraph{S4FFNO.}
S4FFNO instantiates the Memory Neural Operator framework by combining FFNO
layers with an S4 memory layer~\citep{S4FFNO,gu2022efficiently}. All layers
except the memory layer are the same as FFNO. The S4 layer uses state dimension
64 and the diagonal S4D kernel. Following the S4FFNO implementation, the memory
layer operates along the temporal history dimension and is applied
independently at each spatial location. The layer configuration is
\textsc{ffno--ffno--s4--ffno--ffno}.

\paragraph{Multi-Input FFNO.}
Multi-Input FFNO uses the last $K=4$ timesteps of the solution as input to
predict the next timestep, following the multi-step input setting of FNO and
S4FFNO~\citep{li2021fourier,S4FFNO}. The number of FFNO layers and hidden
dimensions are the same as FFNO. The only architectural change is the input
lifting layer, which receives a concatenation of four previous states.

\paragraph{Factformer.}
Factformer is a transformer-based neural operator baseline~\citep{factformer}.
We use the one-dimensional adaptation described in S4FFNO. The model uses four
linear attention layers over the spatial sequence. The hidden dimension is 64,
with four heads and hidden dimension 128 in each attention layer.

\paragraph{Galerkin Transformer.}
The Galerkin Transformer baseline follows~\citet{cao2021choose}. It uses four linear
attention layers over the spatial sequence, concatenates positional information
into the attention inputs, and uses two FNO layers as a spectral regressor. The
hidden dimension is 32.

\paragraph{U-Net.}
The U-Net neural operator baseline follows the U-Net style architecture used for
PDE surrogate modeling~\citep{unet_pde}. It consists of four downsample
convolution blocks, a middle convolution block, and four upsample convolution
blocks with skip connections. The channel multipliers are $[1,2,2,2]$, and no
time embeddings are used.

\paragraph{AMGFNO.}
AMGFNO uses the same FFNO backbone and S4 memory branch as S4FFNO, but replaces
the fixed memory weight with the adaptive memory gate defined in
Section~\ref{subsec:gate}. The gate combines a content-dependent component with
the condition-level spectral scale derived from $\widehat\omega_{f,\rho}$. We
do not redefine the gate here to avoid duplicating Section~\ref{subsec:gate};
this appendix specifies the shared experimental protocol.

\subsection{Training Details}
\label{app:exp-training-details}

All one-dimensional experiments are trained with teacher forcing, following FFNO
and S4FFNO~\citep{FFNO,S4FFNO}. During training, the input for predicting
timestep $j+1$ is the ground-truth resolved state at timestep $j$, rather than
the model's previous prediction. For model $\mathcal G_\theta$, the one-step
prediction is
\begin{equation}
    \widetilde u_{j+1}^{f,(i)}
    =
    \mathcal G_\theta
    \left(
        u_j^{f,(i)},
        g_f
    \right),
    \label{eq:app-one-step}
\end{equation}
where $g_f$ denotes the spatial grid encoding. The mini-batch training loss is
\begin{equation}
    \mathcal L_{\rm train}(\theta)
    =
    \frac{1}{|\mathcal B|T_{\rm pred}}
    \sum_{i\in\mathcal B}
    \sum_{j=0}^{T_{\rm pred}-1}
    \left\|
        \mathcal G_\theta
        \left(
            u_j^{f,(i)},
            g_f
        \right)
        -
        u_{j+1}^{f,(i)}
    \right\|_2^2.
    \label{eq:app-training-loss}
\end{equation}

All models are trained for 200 epochs with initial learning rate
\begin{equation}
    10^{-3},
\end{equation}
using cosine annealing learning-rate scheduling~\citep{loshchilov2017sgdr}.
The batch size is 32 for KS and 64 for Burgers'. Both KS and Burgers'
experiments use 2048 training samples. The maximum memory length is the full
trajectory length used in the rollout: 25 prediction timesteps for KS and 20
prediction timesteps for Burgers'.

\begin{table}[t]
\centering
\caption{Training hyperparameters used in the one-dimensional experiments.}
\label{tab:app-training-hparams}
\small
\setlength{\tabcolsep}{6pt}
\begin{tabular}{lcc}
\toprule
Hyperparameter & KS & Burgers' \\
\midrule
Epochs & 200 & 200 \\
Initial learning rate & $10^{-3}$ & $10^{-3}$ \\
Scheduler & cosine annealing & cosine annealing \\
Batch size & 32 & 64 \\
Training samples & 2048 & 2048 \\
Rollout steps & 25 & 20 \\
\bottomrule
\end{tabular}
\end{table}

We ran our experiments on NVIDIA A6000 GPUs. The one-dimensional experiments
require less than 10GB of GPU memory and take approximately 1--2 hours per run,
depending on the resolution, following the computational condition reported by
S4FFNO.

\subsection{Evaluation}
\label{app:exp-evaluation}

At test time, all models are evaluated autoregressively from the first observed
state. The rollout is initialized by
\begin{equation}
    \widetilde u_0^{f,(i)}
    =
    u_0^{f,(i)},
\end{equation}
and predictions are fed back into the model:
\begin{equation}
    \widetilde u_{j+1}^{f,(i)}
    =
    \mathcal G_\theta
    \left(
        \widetilde u_j^{f,(i)},
        g_f
    \right),
    \qquad
    j=0,\ldots,T_{\rm pred}-1.
    \label{eq:app-rollout}
\end{equation}
Thus, the reported metric evaluates accumulated rollout error rather than
one-step prediction error.

The main evaluation metric is normalized root mean squared error, nRMSE. For
trajectory $i$ and timestep $j$, define
\begin{equation}
    {\rm nRMSE}_{j}^{(i)}(f)
    =
    \frac{
        \left\|
            \widetilde u_j^{f,(i)}
            -
            u_j^{f,(i)}
        \right\|_2
    }{
        \left\|
            u_j^{f,(i)}
        \right\|_2
    }.
    \label{eq:app-step-nrmse}
\end{equation}
The final reported score is averaged over test trajectories and rollout
timesteps:
\begin{equation}
    {\rm nRMSE}(f)
    =
    \frac{1}{N_{\rm test}T_{\rm pred}}
    \sum_{i=1}^{N_{\rm test}}
    \sum_{j=1}^{T_{\rm pred}}
    {\rm nRMSE}_{j}^{(i)}(f).
    \label{eq:app-avg-nrmse}
\end{equation}
The values in Table~\ref{tab:main} are computed using this rollout nRMSE.

For AMGFNO, we additionally measure the mean gate value during rollout. For a
test trajectory, let $g_j^{(i)}\in(0,1)^{f\times H}$ denote the adaptive memory
gate at timestep $j$. We define
\begin{equation}
    \overline g(f,\rho)
    =
    \frac{1}{N_{\rm test}T_{\rm pred}fH}
    \sum_{i=1}^{N_{\rm test}}
    \sum_{j=0}^{T_{\rm pred}-1}
    \left\|
        g_j^{(i)}
    \right\|_1.
    \label{eq:app-mean-gate}
\end{equation}
This value quantifies how strongly AMGFNO uses the memory path under a given
resolution and physical condition. The values reported in Table~\ref{tab:omega_gate}
are computed as the mean gate value over the last 20 epochs.    

\section{Detailed Theory and Proofs for AMGFNO}
\label{app:theory}

This appendix provides the formal details behind the theoretical
motivation in Section~\ref{sec:theory}. The main text keeps the
high-level argument and the frequency-gate quantities
$\omega_f$, $\omega^{\rm pop}_{f,\rho}$, $\widehat\omega_{f,\rho}$,
$\chi_{f,\rho}$, and $\lambda_{f,\rho}$ explicit, while deferring the
full prediction-target, hypothesis-class, approximation-error, and
finite-sample definitions to this appendix.

The comparison treats FFNO, S4FFNO, and AMGFNO as functions of the same
resolved history. FFNO is the restricted subclass that uses only the
current resolved state, S4FFNO uses a fixed residual memory fusion, and
AMGFNO uses the same S4 memory branch with an adaptive memory gate. This
shared-history view separates two effects: the approximation benefit of
memory and the finite-sample complexity cost of maintaining a memory
branch.

The notation is aligned with the AMGFNO gate in
Section~\ref{subsec:gate}. We use $u_t$ for the PDE state, $h_t$ for the
Markovian hidden feature at the AMG insertion point, $z_t$ for the S4
memory output, and $g_t$ for the adaptive memory gate. The symbol $\mathcal F$
is reserved for the Fourier transform in FFNO layers, so predictor
classes are denoted by $\mathcal C$ and loss classes by $\mathfrak L$.
The architectural setting follows FFNO-style neural operators and
S4-based temporal memory models
\citep{gu2020hippo,FFNO,gu2022efficiently,S4FFNO}.

All architecture classes below are identified with their $L^2$ closures
under the population distribution induced by the data-generating process.
This convention lets us use limiting cases such as zero memory output and
zero gate in approximation arguments, even though sigmoid gates take
values in $(0,1)$ before closure.

\paragraph{Proof roadmap.}
Appendix~\ref{app:proof-thm21} proves
Theorem~\ref{thm:gain-convergence}. The proof first bounds the influence
of the unresolved Fourier tail, then transfers approximation of the
projected Markovian target to approximation of the history-conditioned
Bayes target. Appendix~\ref{app:proof-thm22} proves
Theorem~\ref{thm:main-amg-highres}. The proof combines spectral gate
convergence, finite-sample Rademacher bounds, and a sufficient condition
under which the AMGFNO complexity gap closes as its gate budget shrinks.

\subsection{Setup and Common Bayes Target}
\label{app:theory-setup}

We state the argument on the one-dimensional periodic domain $\mathbb T$.
The same notation applies to $[0,L]$ after replacing $e^{in\xi}$ by
$e^{i2\pi n\xi/L}$. Let
\[
    \mathcal H:=L^2(\mathbb T;\mathbb R^V).
\]
For a physical condition $\rho$, let
\begin{equation}
    u_{t+1}^{\rho}
    =
    \Phi_{\Delta t}^{\rho}(u_t^\rho),
    \qquad
    u_0^\rho\sim\mu_\rho,
    \label{eq:app-dynamics}
\end{equation}
where $\Phi_{\Delta t}^{\rho}$ is the time-$\Delta t$ solution map. For
$\phi\in\mathcal H$, write
\begin{equation}
    \phi(\xi)=\sum_{n\in\mathbb Z}a_n(\phi)e^{in\xi}.
\end{equation}
For vector-valued states, $a_n(\phi)\in\mathbb C^V$ and $|a_n(\phi)|^2$
denotes the squared Euclidean norm over channels.

Let $m_f:=\lfloor f/2\rfloor$. The ideal spectral observation at
resolution $f$ is
\begin{equation}
    P_f\phi
    :=
    \sum_{|n|\le m_f}a_n(\phi)e^{in\xi},
    \qquad
    Q_f:=I-P_f,
    \qquad
    \mathcal V_f:=P_f\mathcal H.
    \label{eq:app-projection}
\end{equation}
Thus $Q_f\phi$ is the unresolved Fourier tail. The cutoff $|n|>m_f$
follows the Nyquist--Shannon interpretation of which Fourier modes are
unavailable at resolution $f$ \citep{shannon1949communication}. The
theory uses this idealized spectral projection to isolate the
unresolved-tail mechanism.

The resolved state, resolved history, and one-step resolved target are
\begin{equation}
    u_{t,f}^{\rho}:=P_f u_t^\rho,
    \qquad
    \mathcal U_t^{f,\rho}:=(u_{0,f}^{\rho},\ldots,u_{t,f}^{\rho}),
    \qquad
    Y_t^{f,\rho}:=P_fu_{t+1}^{\rho}.
    \label{eq:app-resolved-objects}
\end{equation}
For clarity, the statements below are written for a fixed finite time
index $t$. If $t$ is sampled uniformly from a finite prediction horizon,
all expectations are replaced by finite averages over $t$, and the
arguments are unchanged.

For a measurable predictor $\psi:\mathcal U_t^{f,\rho}\mapsto\mathcal V_f$,
define the population risk
\begin{equation}
    R_{f,\rho}(\psi)
    :=
    \mathbb E
    \left\|
        Y_t^{f,\rho}
        -
        \psi(\mathcal U_t^{f,\rho})
    \right\|_{\mathcal V_f}^2.
    \label{eq:app-risk}
\end{equation}
The projected Markovian target is
\begin{equation}
    \mathcal T_{f,\rho}(\phi)
    :=
    P_f\Phi_{\Delta t}^{\rho}(\phi),
    \qquad
    \phi\in\mathcal V_f.
    \label{eq:app-projected-target}
\end{equation}
The resolved target decomposes as
\begin{equation}
    Y_t^{f,\rho}
    =
    \mathcal T_{f,\rho}(u_{t,f}^{\rho})
    +
    \delta_{t,f,\rho},
    \label{eq:app-residual-decomposition}
\end{equation}
where
\begin{equation}
    \delta_{t,f,\rho}
    :=
    P_f\Phi_{\Delta t}^{\rho}(u_t^\rho)
    -
    P_f\Phi_{\Delta t}^{\rho}(P_f u_t^\rho).
    \label{eq:app-delta}
\end{equation}
The residual $\delta_{t,f,\rho}$ is the effect of the unresolved tail
$Q_f u_t^\rho$ on the next resolved state.

\begin{proposition}[Bayes decomposition]
\label{prop:bayes-decomposition-final}
Define the history-conditioned Bayes target
\begin{equation}
    b_{f,\rho}(\mathcal U_t^{f,\rho})
    :=
    \mathbb E
    \left[
        Y_t^{f,\rho}
        \mid
        \mathcal U_t^{f,\rho}
    \right]
    \label{eq:app-bayes-target}
\end{equation}
and the irreducible error
\begin{equation}
    \varepsilon_{f,\rho}^2
    :=
    \mathbb E
    \left\|
        Y_t^{f,\rho}
        -
        b_{f,\rho}(\mathcal U_t^{f,\rho})
    \right\|_{\mathcal V_f}^2.
\end{equation}
For every square-integrable predictor $\psi$,
\begin{equation}
    R_{f,\rho}(\psi)
    =
    \varepsilon_{f,\rho}^2
    +
    \mathbb E
    \left\|
        b_{f,\rho}(\mathcal U_t^{f,\rho})
        -
        \psi(\mathcal U_t^{f,\rho})
    \right\|_{\mathcal V_f}^2.
    \label{eq:app-bayes-decomposition}
\end{equation}
\end{proposition}

\begin{proof}
Write $U=\mathcal U_t^{f,\rho}$, $Y=Y_t^{f,\rho}$, and
$b(U)=\mathbb E[Y\mid U]$. Then
\[
    Y-\psi(U)
    =
    \bigl(Y-b(U)\bigr)
    +
    \bigl(b(U)-\psi(U)\bigr).
\]
Expanding the squared norm gives
\begin{align}
    \mathbb E\|Y-\psi(U)\|_{\mathcal V_f}^2
    &=
    \mathbb E\|Y-b(U)\|_{\mathcal V_f}^2
    +
    \mathbb E\|b(U)-\psi(U)\|_{\mathcal V_f}^2 \nonumber\\
    &\quad+
    2\mathbb E
    \left\langle
        Y-b(U),
        b(U)-\psi(U)
    \right\rangle_{\mathcal V_f}.
\end{align}
The cross term is zero because $b(U)-\psi(U)$ is $\sigma(U)$-measurable and
\[
    \mathbb E[Y-b(U)\mid U]
    =
    \mathbb E[Y\mid U]-b(U)
    =
    0.
\]
This proves the decomposition.
\end{proof}

\subsection{Architecture Classes and Approximation Gains}
\label{app:theory-classes}

Let
\[
    h_{k,\theta}:=h_k^{(\ell^*-1)}\in\mathbb R^{S\times H}
\]
be the hidden feature produced from $u_{k,f}^{\rho}$ by the encoder and
the FFNO layers before the AMG insertion point. Let
\begin{equation}
    z_{t,\theta}
    :=
    M_\theta(h_{0,\theta},\ldots,h_{t,\theta})
    \in\mathbb R^{S\times H}
    \label{eq:app-s4-memory}
\end{equation}
be the S4 memory output. The S4 branch applies a state-space recurrence
along the temporal axis at each spatial position, treating the spatial
dimension as a batch axis. Let $\mathrm{Dec}_\theta$ denote the remaining
FFNO layers and output decoder.

The FFNO class is
\begin{equation}
    \mathcal C^F_{f,\rho}
    :=
    \left\{
        \psi_\theta(\mathcal U_t^{f,\rho})
        =
        \mathrm{Dec}_\theta(h_{t,\theta})
    \right\}.
    \label{eq:app-class-ffno}
\end{equation}
This class is formally defined on the resolved history but ignores all
states except the current one.

The fixed-fusion S4FFNO class is
\begin{equation}
    \mathcal C^{S4}_{f,\rho}
    :=
    \left\{
        \psi_\theta(\mathcal U_t^{f,\rho})
        =
        \mathrm{Dec}_\theta(h_{t,\theta}+z_{t,\theta})
    \right\}.
    \label{eq:app-class-s4}
\end{equation}

For the content-gated ablation without the frequency-aware gate, define
\begin{equation}
    \mathcal C^{CG}_{f,\rho}
    :=
    \left\{
        \psi_\theta(\mathcal U_t^{f,\rho})
        =
        \mathrm{Dec}_\theta
        \left(
            (1-g_{t,\theta}^{\rm cont})\odot h_{t,\theta}
            +
            g_{t,\theta}^{\rm cont}\odot z_{t,\theta}
        \right)
    \right\},
    \label{eq:app-class-cg}
\end{equation}
where
\begin{equation}
    g_{t,\theta}^{\rm cont}
    :=
    \sigma(W_z z_{t,\theta}+W_h h_{t,\theta}+b)
    \in(0,1)^{S\times H}.
\end{equation}

For a full reference state $\phi\in\mathcal H$, define the unresolved
spectral-energy ratio
\begin{equation}
    \omega_f(\phi)
    :=
    \frac{
        \sum_{|n|>m_f}|a_n(\phi)|^2
    }{
        \sum_{n\in\mathbb Z}|a_n(\phi)|^2
    },
    \label{eq:app-omega-state}
\end{equation}
with value zero when the denominator is zero. For condition $(f,\rho)$,
define
\begin{equation}
    \omega^{\rm pop}_{f,\rho}
    :=
    \mathbb E_{u_0,t}\left[\omega_f(u_t^\rho)\right].
    \label{eq:app-omega-pop}
\end{equation}
In implementation, $\omega^{\rm pop}_{f,\rho}$ is replaced by the
empirical training-set estimate $\widehat\omega_{f,\rho}$ computed from
high-resolution training trajectories, with Fourier coefficients
approximated by the discrete Fourier transform \citep{cooley1969finite}.
The notation $\omega_f$ in Section~\ref{subsec:gate} denotes this
condition-level scalar when $(f,\rho)$ is fixed. It is not computed from
unresolved test-time modes.

Define the scalar frequency gate by
\begin{equation}
    \chi_{f,\rho}
    :=
    \begin{cases}
    \sigma\!\left(w_1\log\omega^{\rm pop}_{f,\rho}+b_\omega\right),
        & \omega^{\rm pop}_{f,\rho}>0,\\
    0,
        & \omega^{\rm pop}_{f,\rho}=0,
    \end{cases}
    \qquad
    w_1>0.
    \label{eq:app-chi}
\end{equation}
The numerical implementation evaluates
$\log(\widehat\omega_{f,\rho}+\varepsilon_{\rm num})$ for finite-precision
stability. The asymptotic analysis uses the zero-tail convention in
\eqref{eq:app-chi}.

The AMGFNO gate and predictor are
\begin{equation}
    g_{t,\theta}
    :=
    \chi_{f,\rho}g_{t,\theta}^{\rm cont},
    \label{eq:app-amg-gate}
\end{equation}
and
\begin{equation}
    \psi_\theta(\mathcal U_t^{f,\rho})
    =
    \mathrm{Dec}_\theta
    \left(
        (1-g_{t,\theta})\odot h_{t,\theta}
        +
        g_{t,\theta}\odot z_{t,\theta}
    \right).
    \label{eq:app-amg-predictor}
\end{equation}
The corresponding AMGFNO class is denoted by $\mathcal C^{AMG}_{f,\rho}$.

For $\psi\in\mathcal C^{AMG}_{f,\rho}$, define the normalized gate budget
\begin{equation}
    \Lambda_{f,\rho}(\psi)
    :=
    \mathbb E
    \left[
        \frac{1}{SH}\|g_{t,\theta}\|_F^2
    \right].
    \label{eq:app-gate-budget}
\end{equation}
For $\lambda\ge0$, define
\begin{equation}
    \mathcal C^{AMG}_{f,\rho}(\lambda)
    :=
    \left\{
        \psi\in\mathcal C^{AMG}_{f,\rho}
        :
        \Lambda_{f,\rho}(\psi)\le\lambda
    \right\}.
    \label{eq:app-amg-budget-class}
\end{equation}
Since $0\le g_{t,\theta}^{\rm cont}\le1$ entrywise,
\begin{equation}
    \Lambda_{f,\rho}(\psi)
    \le
    \chi_{f,\rho}^2.
\end{equation}
We write
\begin{equation}
    \lambda_{f,\rho}
    :=
    \chi_{f,\rho}^2.
    \label{eq:app-lambda}
\end{equation}

\begin{proposition}[Markovian fallback]
\label{prop:markovian-fallback-final}
Assume the S4 parameterization contains a zero-output realization, for
example by setting its readout to zero. Under the $L^2$-closure
convention,
\begin{equation}
    \mathcal C^F_{f,\rho}\subset\mathcal C^{S4}_{f,\rho},
    \qquad
    \mathcal C^F_{f,\rho}\subset\mathcal C^{CG}_{f,\rho},
    \qquad
    \mathcal C^F_{f,\rho}\subset\mathcal C^{AMG}_{f,\rho}(\lambda)
\end{equation}
for every $\lambda\ge0$.
\end{proposition}

\begin{proof}
For S4FFNO, choose the S4 readout so that $z_{t,\theta}=0$ for every
history. Then
\[
    \mathrm{Dec}_\theta(h_{t,\theta}+z_{t,\theta})
    =
    \mathrm{Dec}_\theta(h_{t,\theta}),
\]
which is an FFNO predictor.

For the gated classes, take a sequence of gate biases tending to
$-\infty$, so that the gate converges to zero in $L^2$. The fused feature
converges to
\[
    (1-g_{t,\theta})\odot h_{t,\theta}
    +
    g_{t,\theta}\odot z_{t,\theta}
    \to
    h_{t,\theta}.
\]
The resulting predictor is $\mathrm{Dec}_\theta(h_{t,\theta})$, and the
gate budget is zero. Hence the FFNO predictor belongs to the gated
classes in closure.
\end{proof}

Define the approximation errors to the common Bayes target:
\begin{align}
    \mathcal E^F_{f,\rho}
    &:=
    \inf_{\psi\in\mathcal C^F_{f,\rho}}
    \mathbb E
    \left\|
        b_{f,\rho}(\mathcal U_t^{f,\rho})
        -
        \psi(\mathcal U_t^{f,\rho})
    \right\|_{\mathcal V_f}^2,
    \label{eq:app-E-F}\\
    \mathcal E^{S4}_{f,\rho}
    &:=
    \inf_{\psi\in\mathcal C^{S4}_{f,\rho}}
    \mathbb E
    \left\|
        b_{f,\rho}(\mathcal U_t^{f,\rho})
        -
        \psi(\mathcal U_t^{f,\rho})
    \right\|_{\mathcal V_f}^2,\\
    \mathcal E^{CG}_{f,\rho}
    &:=
    \inf_{\psi\in\mathcal C^{CG}_{f,\rho}}
    \mathbb E
    \left\|
        b_{f,\rho}(\mathcal U_t^{f,\rho})
        -
        \psi(\mathcal U_t^{f,\rho})
    \right\|_{\mathcal V_f}^2,\\
    \mathcal E^{AMG}_{f,\rho}(\lambda)
    &:=
    \inf_{\psi\in\mathcal C^{AMG}_{f,\rho}(\lambda)}
    \mathbb E
    \left\|
        b_{f,\rho}(\mathcal U_t^{f,\rho})
        -
        \psi(\mathcal U_t^{f,\rho})
    \right\|_{\mathcal V_f}^2.
    \label{eq:app-E-AMG}
\end{align}
The approximation gains over FFNO are
\begin{align}
    \Gamma^{S4}_{f,\rho}
    &:=
    \mathcal E^F_{f,\rho}
    -
    \mathcal E^{S4}_{f,\rho},\\
    \Gamma^{CG}_{f,\rho}
    &:=
    \mathcal E^F_{f,\rho}
    -
    \mathcal E^{CG}_{f,\rho},\\
    \Gamma^{AMG}_{f,\rho}(\lambda)
    &:=
    \mathcal E^F_{f,\rho}
    -
    \mathcal E^{AMG}_{f,\rho}(\lambda).
\end{align}
By Proposition~\ref{prop:markovian-fallback-final}, these gains are
nonnegative.

Finally, define the FFNO approximation error to the projected Markovian
target:
\begin{equation}
    \mathcal A^F_{f,\rho}
    :=
    \inf_{\psi\in\mathcal C^F_{f,\rho}}
    \mathbb E
    \left\|
        \mathcal T_{f,\rho}(u_{t,f}^{\rho})
        -
        \psi(\mathcal U_t^{f,\rho})
    \right\|_{\mathcal V_f}^2.
    \label{eq:app-A-F}
\end{equation}

\paragraph{Why distinguish $\mathcal A^F_{f,\rho}$ and
$\mathcal E^F_{f,\rho}$?}
The distinction is central to the argument.
$\mathcal A^F_{f,\rho}$ measures how well FFNO can approximate the ideal
projected Markovian target $\mathcal T_{f,\rho}(u^\rho_{t,f})$, which
depends only on the current resolved state. This is the approximation
quantity appearing in Assumption~\ref{ass:ffno-projected-approx}. In
contrast, $\mathcal E^F_{f,\rho}$ measures FFNO's approximation error to
the history-conditioned Bayes target
$b_{f,\rho}(\mathcal U_t^{f,\rho})$, which is the actual partially
observed prediction target and may depend on the entire resolved history.
Therefore, the theory does not assume that FFNO already approximates the
history-conditioned Bayes target. Instead,
Lemma~\ref{lem:bayes-target-convergence-final} uses the unresolved-tail
residual bound to transfer approximation of the projected Markovian
target into a bound on the Bayes-target error.

\subsection{Assumptions}
\label{app:theory-assumptions}

\begin{assumption}[Finite-horizon well-posedness and Sobolev bound]
\label{ass:wellposed-sobolev}
For $\mu_\rho$-almost every initial condition, the PDE admits a unique
solution over the finite prediction horizon. For some $r\ge0$ and
$\beta>0$, there exists $R_{\rho,\beta}<\infty$ such that
\begin{equation}
    \sup_{0\le t\le T_{\rm pred}}
    \|u_t^\rho\|_{H^{r+\beta}(\mathbb T)}
    \le
    R_{\rho,\beta}
    \qquad\text{almost surely}.
    \label{eq:ass-sobolev-bound}
\end{equation}
For dissipative long-time settings, this may be replaced by an
absorbing-ball bound in $H^{r+\beta}$. Such well-posedness,
finite-horizon regularity, and dissipativity assumptions are standard for
semilinear parabolic and dissipative PDEs
\citep{henry1981geometric,pazy1983semigroups,amann1984existence,temam1997infinite,robinson2001infinite}.
For the Kuramoto--Sivashinsky equation, they are supported by the KS
well-posedness and attractor literature
\citep{nicolaenko1985global,tadmor1986wellposed}.
\end{assumption}

\begin{assumption}[Local Lipschitz continuity of the solution map]
\label{ass:solution-lipschitz}
For every $R>0$, there exists $L_{\rho,R}<\infty$ such that
\begin{equation}
    \|\Phi_{\Delta t}^{\rho}(\phi)-\Phi_{\Delta t}^{\rho}(\psi)\|_{L^2(\mathbb T)}
    \le
    L_{\rho,R}\|\phi-\psi\|_{H^r(\mathbb T)}
    \label{eq:ass-lipschitz}
\end{equation}
whenever $\|\phi\|_{H^r}\le R$ and $\|\psi\|_{H^r}\le R$. This is the
finite-time continuous-dependence property used in semigroup treatments
of semilinear parabolic equations
\citep{henry1981geometric,pazy1983semigroups,amann1984existence}.
\end{assumption}

\begin{assumption}[FFNO approximation of the projected Markovian target]
\label{ass:ffno-projected-approx}
The projected-target approximation error satisfies
\begin{equation}
    \mathcal A^F_{f,\rho}\to0
    \qquad
    \text{as }f\to\infty.
    \label{eq:ass-ffno-projected-approx}
\end{equation}
This assumption concerns only the ideal projected Markovian target
$\mathcal T_{f,\rho}$, not the history-conditioned Bayes target
$b_{f,\rho}$. Thus it is not an assumption that FFNO can already solve
the partially observed prediction problem. Rather, it states that, as the
resolved space becomes richer, the FFNO class can approximate the
Markovian solution map restricted to the resolved variables. This is
consistent with neural-operator approximation results for continuous
solution operators on compact sets
\citep{li2021fourier,FFNO,kovachki2023neuraloperator}.
\end{assumption}

\begin{assumption}[Bounded or sub-Gaussian losses]
\label{ass:bounded-loss}
For the finite-sample bounds, assume squared losses are uniformly bounded
by $\ell_\infty$. Equivalently, one may replace this by a standard
sub-Gaussian condition and obtain corresponding high-probability
concentration terms. The bounded-loss form is used only to keep the
comparison concise \citep{Bach}.
\end{assumption}

\begin{assumption}[Trajectory-level sampling]
\label{ass:trajectory-sampling}
Training trajectories are independent. Multiple timesteps from the same
trajectory may be dependent. The finite-sample statements treat each
trajectory as one independent sample and define the empirical loss as the
average over the finite prediction horizon. This follows the standard
multi-trajectory learning setup, where trajectories are independent but
covariates within a trajectory need not be independent~\citep{tu2024learning}.
\end{assumption}

\subsection{Residual Decay and Proof of Theorem~\ref{thm:gain-convergence}}
\label{app:theory-convergence}
\label{app:proof-thm21}

\begin{lemma}[Residual decay]
\label{lem:residual-decay-final}
Under Assumptions~\ref{ass:wellposed-sobolev} and
\ref{ass:solution-lipschitz}, there exists a constant
$\kappa_{{\rm tail},\rho}>0$, independent of $f$, such that
\begin{equation}
    \mathbb E
    \|\delta_{t,f,\rho}\|_{\mathcal V_f}^2
    \le
    \kappa_{{\rm tail},\rho}f^{-2\beta}.
    \label{eq:app-residual-decay}
\end{equation}
\end{lemma}

\begin{proof}
Since $P_f$ is an orthogonal projection in $L^2$,
\begin{align}
    \|\delta_{t,f,\rho}\|_{\mathcal V_f}
    &=
    \left\|
        P_f
        \left(
            \Phi_{\Delta t}^{\rho}(u_t^\rho)
            -
            \Phi_{\Delta t}^{\rho}(P_f u_t^\rho)
        \right)
    \right\|_{L^2} \nonumber\\
    &\le
    \left\|
        \Phi_{\Delta t}^{\rho}(u_t^\rho)
        -
        \Phi_{\Delta t}^{\rho}(P_f u_t^\rho)
    \right\|_{L^2}.
    \label{eq:proof-projection-nonexpansive}
\end{align}
By Assumption~\ref{ass:wellposed-sobolev},
$\|u_t^\rho\|_{H^{r+\beta}}\le R_{\rho,\beta}$ almost surely. Since
$P_f$ is nonexpansive in Sobolev norms,
\[
    \|P_f u_t^\rho\|_{H^r}
    \le
    \|u_t^\rho\|_{H^r}
    \le
    \|u_t^\rho\|_{H^{r+\beta}}
    \le
    R_{\rho,\beta}.
\]
Thus both $u_t^\rho$ and $P_f u_t^\rho$ lie in the same $H^r$ ball.
Applying Assumption~\ref{ass:solution-lipschitz} gives
\begin{equation}
    \|\delta_{t,f,\rho}\|_{\mathcal V_f}
    \le
    L_{\rho,R_{\rho,\beta}}
    \|u_t^\rho-P_f u_t^\rho\|_{H^r}
    =
    L_{\rho,R_{\rho,\beta}}
    \|Q_f u_t^\rho\|_{H^r}.
    \label{eq:proof-lipschitz-tail}
\end{equation}

Using the Fourier definition of the $H^r$ norm,
\begin{align}
    \|Q_f u_t^\rho\|_{H^r}^2
    &=
    \sum_{|n|>m_f}
    (1+|n|^2)^r
    |a_n(u_t^\rho)|^2 \nonumber\\
    &=
    \sum_{|n|>m_f}
    (1+|n|^2)^{r+\beta}
    (1+|n|^2)^{-\beta}
    |a_n(u_t^\rho)|^2 .
\end{align}
For $|n|>m_f$, there exists $C_\beta$ independent of $f$ such that
\[
    (1+|n|^2)^{-\beta}
    \le
    (1+m_f^2)^{-\beta}
    \le
    C_\beta f^{-2\beta}.
\]
Therefore
\begin{equation}
    \|Q_f u_t^\rho\|_{H^r}^2
    \le
    C_\beta f^{-2\beta}
    \|u_t^\rho\|_{H^{r+\beta}}^2
    \le
    C_\beta f^{-2\beta}R_{\rho,\beta}^2.
\end{equation}
Combining this with \eqref{eq:proof-lipschitz-tail} and taking
expectations proves the claim with
\[
    \kappa_{{\rm tail},\rho}
    :=
    C_\beta L_{\rho,R_{\rho,\beta}}^2R_{\rho,\beta}^2.
\]
\end{proof}

\begin{lemma}[From projected-target approximation to Bayes-target approximation]
\label{lem:bayes-target-convergence-final}
Under Assumptions~\ref{ass:wellposed-sobolev} and
\ref{ass:solution-lipschitz},
\begin{equation}
    \mathcal E^F_{f,\rho}
    \le
    2\mathcal A^F_{f,\rho}
    +
    2\kappa_{{\rm tail},\rho}f^{-2\beta}.
    \label{eq:app-E-bound-by-A}
\end{equation}
If Assumption~\ref{ass:ffno-projected-approx} also holds, then
$\mathcal E^F_{f,\rho}\to0$ as $f\to\infty$.
\end{lemma}

\begin{proof}
The proof transfers approximation of the projected Markovian target to
approximation of the true history-conditioned Bayes target. From
\eqref{eq:app-residual-decomposition},
\[
    Y_t^{f,\rho}
    =
    \mathcal T_{f,\rho}(u_{t,f}^{\rho})
    +
    \delta_{t,f,\rho}.
\]
Taking conditional expectation with respect to
$\mathcal U_t^{f,\rho}$ gives
\begin{equation}
    b_{f,\rho}(\mathcal U_t^{f,\rho})
    =
    \mathcal T_{f,\rho}(u_{t,f}^{\rho})
    +
    c_{t,f,\rho}(\mathcal U_t^{f,\rho}),
    \label{eq:app-bayes-target-residual}
\end{equation}
where
\[
    c_{t,f,\rho}(\mathcal U_t^{f,\rho})
    :=
    \mathbb E
    \left[
        \delta_{t,f,\rho}
        \mid
        \mathcal U_t^{f,\rho}
    \right].
\]
The term $\mathcal T_{f,\rho}(u_{t,f}^{\rho})$ is
$\sigma(\mathcal U_t^{f,\rho})$-measurable.

For any $\psi\in\mathcal C^F_{f,\rho}$,
\begin{align}
    &
    \mathbb E
    \left\|
        b_{f,\rho}(\mathcal U_t^{f,\rho})
        -
        \psi(\mathcal U_t^{f,\rho})
    \right\|_{\mathcal V_f}^2
    \nonumber\\
    &\quad\le
    2
    \mathbb E
    \left\|
        \mathcal T_{f,\rho}(u_{t,f}^{\rho})
        -
        \psi(\mathcal U_t^{f,\rho})
    \right\|_{\mathcal V_f}^2
    +
    2
    \mathbb E
    \left\|
        c_{t,f,\rho}(\mathcal U_t^{f,\rho})
    \right\|_{\mathcal V_f}^2.
    \label{eq:app-bayes-bound-intermediate}
\end{align}
By Jensen's inequality for conditional expectation,
\begin{equation}
    \mathbb E
    \left\|
        c_{t,f,\rho}(\mathcal U_t^{f,\rho})
    \right\|_{\mathcal V_f}^2
    =
    \mathbb E
    \left\|
        \mathbb E[\delta_{t,f,\rho}\mid\mathcal U_t^{f,\rho}]
    \right\|_{\mathcal V_f}^2
    \le
    \mathbb E
    \|\delta_{t,f,\rho}\|_{\mathcal V_f}^2.
\end{equation}
Taking the infimum over $\psi\in\mathcal C^F_{f,\rho}$ in
\eqref{eq:app-bayes-bound-intermediate} and applying
Lemma~\ref{lem:residual-decay-final} proves
\eqref{eq:app-E-bound-by-A}. If $\mathcal A^F_{f,\rho}\to0$, the
right-hand side converges to zero.
\end{proof}

\begin{proof}[Proof of Theorem~\ref{thm:gain-convergence}]
Lemma~\ref{lem:bayes-target-convergence-final} gives
\begin{equation}
    \mathcal E^F_{f,\rho}
    \le
    2\mathcal A^F_{f,\rho}
    +
    2\kappa_{{\rm tail},\rho}f^{-2\beta}.
\end{equation}
Under Assumption~\ref{ass:ffno-projected-approx},
$\mathcal A^F_{f,\rho}\to0$. Since $f^{-2\beta}\to0$, this implies
\[
    \mathcal E^F_{f,\rho}\to0.
\]
Importantly, the convergence of $\mathcal E^F_{f,\rho}$ is a conclusion
of the residual-transfer bound, not an assumption.

By Proposition~\ref{prop:markovian-fallback-final}, each
memory-augmented class contains the FFNO fallback in closure. Hence
\[
    \mathcal E^{S4}_{f,\rho}
    \le
    \mathcal E^F_{f,\rho},
    \qquad
    \mathcal E^{AMG}_{f,\rho}(\lambda)
    \le
    \mathcal E^F_{f,\rho}
\]
for every $\lambda\ge0$. Therefore the corresponding memory gains satisfy
\[
    0
    \le
    \Gamma^{S4}_{f,\rho}
    =
    \mathcal E^F_{f,\rho}
    -
    \mathcal E^{S4}_{f,\rho}
    \le
    \mathcal E^F_{f,\rho},
\]
and
\[
    0
    \le
    \Gamma^{AMG}_{f,\rho}(\lambda)
    =
    \mathcal E^F_{f,\rho}
    -
    \mathcal E^{AMG}_{f,\rho}(\lambda)
    \le
    \mathcal E^F_{f,\rho}.
\]
Since $\mathcal E^F_{f,\rho}\to0$, both gains converge to zero by
squeezing. This proves Theorem~\ref{thm:gain-convergence}.
\end{proof}

\subsection{Spectral Gate Convergence}
\label{app:theory-gate-convergence}

\begin{lemma}[Condition-level spectral gate convergence]
\label{lem:gate-convergence-final}
Let $\omega_f(\phi)$, $\omega^{\rm pop}_{f,\rho}$, $\chi_{f,\rho}$, and
$\lambda_{f,\rho}$ be defined by
\eqref{eq:app-omega-state}--\eqref{eq:app-lambda}. Then
\begin{equation}
    \omega^{\rm pop}_{f,\rho}\to0,
    \qquad
    \chi_{f,\rho}\to0,
    \qquad
    \lambda_{f,\rho}=\chi_{f,\rho}^2\to0
    \qquad
    \text{as }f\to\infty.
    \label{eq:app-gate-convergence}
\end{equation}
Moreover, every $\psi\in\mathcal C^{AMG}_{f,\rho}$ satisfies
\begin{equation}
    \Lambda_{f,\rho}(\psi)\le\lambda_{f,\rho}.
\end{equation}
\end{lemma}

\begin{proof}
Fix any $u_t^\rho\in L^2(\mathbb T;\mathbb R^V)$. Its Fourier
coefficients are square summable:
\[
    \sum_{n\in\mathbb Z}|a_n(u_t^\rho)|^2<\infty.
\]
Therefore
\[
    \sum_{|n|>m_f}|a_n(u_t^\rho)|^2
    \to0
    \qquad
    \text{as }f\to\infty.
\]
If the denominator in \eqref{eq:app-omega-state} is positive, this
implies $\omega_f(u_t^\rho)\to0$. If the denominator is zero, then
$\omega_f(u_t^\rho)=0$ for every $f$ by convention. Hence
$\omega_f(u_t^\rho)\to0$ almost surely.

Since $0\le\omega_f(\cdot)\le1$, dominated convergence gives
\[
    \omega^{\rm pop}_{f,\rho}
    =
    \mathbb E_{u_0,t}[\omega_f(u_t^\rho)]
    \to0.
\]
Because $w_1>0$, the quantity
$w_1\log\omega^{\rm pop}_{f,\rho}+b_\omega$ tends to $-\infty$ whenever
$\omega^{\rm pop}_{f,\rho}>0$ along the sequence, while the zero-tail
convention covers the case $\omega^{\rm pop}_{f,\rho}=0$. Thus
$\chi_{f,\rho}\to0$, and $\lambda_{f,\rho}\to0$.

Finally, $g_{t,\theta}=\chi_{f,\rho}g_{t,\theta}^{\rm cont}$ and
$0\le g_{t,\theta}^{\rm cont}\le1$ entrywise. Hence
\[
    \frac{1}{SH}\|g_{t,\theta}\|_F^2
    \le
    \chi_{f,\rho}^2
    =
    \lambda_{f,\rho}.
\]
Taking expectations proves the budget bound.
\end{proof}

\subsection{Finite-sample Bounds}
\label{app:theory-finite-sample}

Following Assumption~\ref{ass:trajectory-sampling}, the independent
sample unit is a trajectory. Let $Z_i$ denote the $i$-th resolved
trajectory, including histories and targets over the finite prediction
horizon. For display, write a single-time loss; the finite-horizon
average can be absorbed into the definition of $\ell_\psi$.

For a predictor $\psi$, define
\begin{equation}
    \ell_\psi(Z)
    :=
    \left\|
        Y_t^{f,\rho}
        -
        \psi(\mathcal U_t^{f,\rho})
    \right\|_{\mathcal V_f}^2.
\end{equation}
For a predictor class $\Theta$, define the induced loss class
\begin{equation}
    \mathfrak L(\Theta)
    :=
    \{Z\mapsto\ell_\psi(Z):\psi\in\Theta\}.
\end{equation}
We write
\begin{equation}
    \mathfrak L^F_{f,\rho}:=\mathfrak L(\mathcal C^F_{f,\rho}),
    \qquad
    \mathfrak L^{S4}_{f,\rho}:=\mathfrak L(\mathcal C^{S4}_{f,\rho}),
\end{equation}
\begin{equation}
    \mathfrak L^{CG}_{f,\rho}:=\mathfrak L(\mathcal C^{CG}_{f,\rho}),
    \qquad
    \mathfrak L^{AMG}_{f,\rho}(\lambda)
    :=
    \mathfrak L(\mathcal C^{AMG}_{f,\rho}(\lambda)).
\end{equation}
Let $R_m(\mathfrak L)$ denote the population Rademacher complexity of a
real-valued loss class $\mathfrak L$ with sample size $m$:
\begin{equation}
    R_m(\mathfrak L)
    :=
    \mathbb E_{Z_{1:m},\epsilon}
    \left[
        \sup_{\ell\in\mathfrak L}
        \frac{1}{m}
        \sum_{i=1}^{m}
        \epsilon_i\ell(Z_i)
    \right],
    \label{eq:app-rademacher}
\end{equation}
where $\epsilon_i$ are independent Rademacher random variables.

Under Assumption~\ref{ass:bounded-loss}, a standard Rademacher ERM bound
implies that, with probability at least $1-\eta$, the empirical risk
minimizer over any class $\Theta$ satisfies \citep{Bach}
\begin{equation}
    R_{f,\rho}(\widehat\psi^\Theta)
    \le
    \inf_{\psi\in\Theta}R_{f,\rho}(\psi)
    +
    4R_m(\mathfrak L(\Theta))
    +
    2\ell_\infty
    \sqrt{\frac{\log(2/\eta)}{2m}}.
    \label{eq:app-standard-erm-bound}
\end{equation}
Applying this bound to the relevant classes and using
Proposition~\ref{prop:bayes-decomposition-final}, define
\begin{align}
    U^F_{f,\rho}(\eta)
    &:=
    \varepsilon_{f,\rho}^2
    +
    \mathcal E^F_{f,\rho}
    +
    4R_m(\mathfrak L^F_{f,\rho})
    +
    2\ell_\infty
    \sqrt{\frac{\log(8/\eta)}{2m}},
    \label{eq:app-U-F}\\
    U^{S4}_{f,\rho}(\eta)
    &:=
    \varepsilon_{f,\rho}^2
    +
    \mathcal E^{S4}_{f,\rho}
    +
    4R_m(\mathfrak L^{S4}_{f,\rho})
    +
    2\ell_\infty
    \sqrt{\frac{\log(8/\eta)}{2m}},\\
    U^{CG}_{f,\rho}(\eta)
    &:=
    \varepsilon_{f,\rho}^2
    +
    \mathcal E^{CG}_{f,\rho}
    +
    4R_m(\mathfrak L^{CG}_{f,\rho})
    +
    2\ell_\infty
    \sqrt{\frac{\log(8/\eta)}{2m}},\\
    U^{AMG}_{f,\rho}(\lambda;\eta)
    &:=
    \varepsilon_{f,\rho}^2
    +
    \mathcal E^{AMG}_{f,\rho}(\lambda)
    +
    4R_m(\mathfrak L^{AMG}_{f,\rho}(\lambda))
    +
    2\ell_\infty
    \sqrt{\frac{\log(8/\eta)}{2m}}.
    \label{eq:app-U-AMG}
\end{align}
The common irreducible Bayes error and concentration terms cancel when
subtracting the FFNO bound:
\begin{align}
    U^{S4}_{f,\rho}(\eta)-U^F_{f,\rho}(\eta)
    &=
    -\Gamma^{S4}_{f,\rho}
    +
    4
    \left(
        R_m(\mathfrak L^{S4}_{f,\rho})
        -
        R_m(\mathfrak L^F_{f,\rho})
    \right),
    \label{eq:app-U-diff-S4}\\
    U^{CG}_{f,\rho}(\eta)-U^F_{f,\rho}(\eta)
    &=
    -\Gamma^{CG}_{f,\rho}
    +
    4
    \left(
        R_m(\mathfrak L^{CG}_{f,\rho})
        -
        R_m(\mathfrak L^F_{f,\rho})
    \right),\\
    U^{AMG}_{f,\rho}(\lambda;\eta)-U^F_{f,\rho}(\eta)
    &=
    -\Gamma^{AMG}_{f,\rho}(\lambda)
    +
    4
    \left(
        R_m(\mathfrak L^{AMG}_{f,\rho}(\lambda))
        -
        R_m(\mathfrak L^F_{f,\rho})
    \right).
    \label{eq:app-U-diff-AMG}
\end{align}
These identities separate the approximation gain of memory from its
effective finite-sample complexity cost.

\begin{theorem}[Low-resolution preference condition]
\label{thm:low-resolution-preference-final}
At any fixed resolution $f$, a memory class has a smaller upper bound
than FFNO whenever its approximation gain exceeds four times its
Rademacher-complexity gap. In particular, S4FFNO satisfies
$U^{S4}_{f,\rho}(\eta)<U^F_{f,\rho}(\eta)$ if
\begin{equation}
    \Gamma^{S4}_{f,\rho}
    >
    4
    \left(
        R_m(\mathfrak L^{S4}_{f,\rho})
        -
        R_m(\mathfrak L^F_{f,\rho})
    \right),
\end{equation}
and AMGFNO satisfies
$U^{AMG}_{f,\rho}(\lambda;\eta)<U^F_{f,\rho}(\eta)$ if
\begin{equation}
    \Gamma^{AMG}_{f,\rho}(\lambda)
    >
    4
    \left(
        R_m(\mathfrak L^{AMG}_{f,\rho}(\lambda))
        -
        R_m(\mathfrak L^F_{f,\rho})
    \right).
\end{equation}
\end{theorem}

\begin{proof}
Both statements follow directly from
\eqref{eq:app-U-diff-S4} and \eqref{eq:app-U-diff-AMG}. If the gain is
larger than four times the complexity gap, the corresponding upper-bound
difference is negative.
\end{proof}

\subsection{A Sufficient Condition for AMGFNO Complexity Convergence}
\label{app:theory-complexity-convergence}

This subsection gives a sufficient condition under which the AMGFNO
loss-class complexity returns to the FFNO loss-class complexity as the
spectral gate closes. The condition is not necessary; it makes the
conditional finite-sample comparison explicit.

Use the normalized feature norm
\begin{equation}
    \|A\|_{S,H}^2
    :=
    \frac{1}{SH}\|A\|_F^2.
\end{equation}

\begin{proposition}[AMGFNO complexity convergence under a closing gate]
\label{prop:amg-complexity-convergence}
Suppose there exist constants
$L_{\rm Dec},B_h,B_z,B_o<\infty$, independent of $f$, such that for all
considered parameters and samples,
\begin{align}
    \|\mathrm{Dec}_\theta(a)-\mathrm{Dec}_\theta(b)\|_{\mathcal V_f}
    &\le
    L_{\rm Dec}\|a-b\|_{S,H},
    \label{eq:prop-decoder-lipschitz}\\
    \|h_{t,\theta}\|_{S,H}
    &\le
    B_h,
    \qquad
    \|z_{t,\theta}\|_{S,H}
    \le
    B_z,
    \label{eq:prop-feature-bounds}\\
    \|Y_t^{f,\rho}\|_{\mathcal V_f}
    &\le
    B_o,
    \qquad
    \|\psi(\mathcal U_t^{f,\rho})\|_{\mathcal V_f}
    \le
    B_o
    \quad
    \text{for all predictors considered}.
    \label{eq:prop-output-bounds}
\end{align}
Then, for $\lambda_{f,\rho}=\chi_{f,\rho}^2$,
\begin{equation}
    0
    \le
    R_m(\mathfrak L^{AMG}_{f,\rho}(\lambda_{f,\rho}))
    -
    R_m(\mathfrak L^F_{f,\rho})
    \le
    C_{\rm comp}\sqrt{\lambda_{f,\rho}},
    \label{eq:app-complexity-convergence}
\end{equation}
where $C_{\rm comp}$ is independent of $f$. Consequently, this complexity
gap converges to zero as $f\to\infty$.
\end{proposition}

\begin{proof}
The lower bound follows from monotonicity of Rademacher complexity
because Proposition~\ref{prop:markovian-fallback-final} gives
$\mathcal C^F_{f,\rho}\subset\mathcal C^{AMG}_{f,\rho}(\lambda_{f,\rho})$
in closure, and therefore
$\mathfrak L^F_{f,\rho}\subset\mathfrak L^{AMG}_{f,\rho}(\lambda_{f,\rho})$
in closure.

For the upper bound, take any
$\psi_\theta^{AMG}\in\mathcal C^{AMG}_{f,\rho}(\lambda_{f,\rho})$ and
define its FFNO fallback with the same encoder, pre-AMG layers, and
decoder:
\begin{equation}
    \psi_\theta^F(\mathcal U_t^{f,\rho})
    :=
    \mathrm{Dec}_\theta(h_{t,\theta}).
\end{equation}
The AMGFNO fused feature differs from this fallback feature by
\begin{align}
    &
    (1-g_{t,\theta})\odot h_{t,\theta}
    +
    g_{t,\theta}\odot z_{t,\theta}
    -
    h_{t,\theta}
    \nonumber\\
    &\qquad
    =
    g_{t,\theta}\odot(z_{t,\theta}-h_{t,\theta}).
\end{align}
Since $g_{t,\theta}=\chi_{f,\rho}g_{t,\theta}^{\rm cont}$ and
$0\le g_{t,\theta}^{\rm cont}\le1$ entrywise,
\[
    \|g_{t,\theta}\|_\infty
    \le
    \chi_{f,\rho}
    =
    \sqrt{\lambda_{f,\rho}}.
\]
Therefore, using \eqref{eq:prop-feature-bounds},
\begin{align}
    \|g_{t,\theta}\odot(z_{t,\theta}-h_{t,\theta})\|_{S,H}
    &\le
    \sqrt{\lambda_{f,\rho}}
    \|z_{t,\theta}-h_{t,\theta}\|_{S,H} \nonumber\\
    &\le
    \sqrt{\lambda_{f,\rho}}(B_h+B_z).
\end{align}
By the decoder Lipschitz condition \eqref{eq:prop-decoder-lipschitz},
\begin{equation}
    \left\|
        \psi_\theta^{AMG}(\mathcal U_t^{f,\rho})
        -
        \psi_\theta^F(\mathcal U_t^{f,\rho})
    \right\|_{\mathcal V_f}
    \le
    L_{\rm Dec}(B_h+B_z)\sqrt{\lambda_{f,\rho}}.
    \label{eq:app-prediction-difference}
\end{equation}

For any target $y$ and predictions $p,q$ with norms bounded by $B_o$,
\begin{align}
    \left|
        \|y-p\|^2-\|y-q\|^2
    \right|
    &=
    \left|
        \langle q-p,2y-p-q\rangle
    \right| \nonumber\\
    &\le
    \|p-q\|
    \left(
        2\|y\|+\|p\|+\|q\|
    \right) \nonumber\\
    &\le
    4B_o\|p-q\|,
\end{align}
where the last step uses \eqref{eq:prop-output-bounds}. Combining this
with \eqref{eq:app-prediction-difference}, there exists
\[
    C_{\rm loss}
    :=
    4B_oL_{\rm Dec}(B_h+B_z)
\]
such that
\begin{equation}
    \left|
        \ell_{\psi_\theta^{AMG}}(Z)
        -
        \ell_{\psi_\theta^F}(Z)
    \right|
    \le
    C_{\rm loss}\sqrt{\lambda_{f,\rho}}
    \label{eq:app-loss-difference}
\end{equation}
for every sample $Z$.

Define the difference class
\begin{equation}
    \mathfrak D_{f,\rho}
    :=
    \left\{
        Z\mapsto
        \ell_{\psi_\theta^{AMG}}(Z)
        -
        \ell_{\psi_\theta^F}(Z)
        :
        \psi_\theta^{AMG}\in
        \mathcal C^{AMG}_{f,\rho}(\lambda_{f,\rho})
    \right\}.
\end{equation}
By \eqref{eq:app-loss-difference}, every function in
$\mathfrak D_{f,\rho}$ has supremum norm at most
$C_{\rm loss}\sqrt{\lambda_{f,\rho}}$, so
\[
    R_m(\mathfrak D_{f,\rho})
    \le
    C_{\rm loss}\sqrt{\lambda_{f,\rho}}.
\]
Each AMGFNO loss can be written as an FFNO fallback loss plus a function
in $\mathfrak D_{f,\rho}$. By subadditivity of Rademacher complexity,
\begin{align}
    R_m(\mathfrak L^{AMG}_{f,\rho}(\lambda_{f,\rho}))
    &\le
    R_m(\mathfrak L^F_{f,\rho})
    +
    R_m(\mathfrak D_{f,\rho}) \nonumber\\
    &\le
    R_m(\mathfrak L^F_{f,\rho})
    +
    C_{\rm loss}\sqrt{\lambda_{f,\rho}}.
\end{align}
Taking $C_{\rm comp}=C_{\rm loss}$ proves the upper bound. Since
$\lambda_{f,\rho}\to0$ by Lemma~\ref{lem:gate-convergence-final}, the
gap converges to zero. The subadditivity and bounded-difference steps are
standard Rademacher-complexity arguments \citep{Bach}.
\end{proof}

\subsection{Proof of Theorem~\ref{thm:main-amg-highres}}
\label{app:theory-main-proof}
\label{app:proof-thm22}

\begin{proof}[Proof of Theorem~\ref{thm:main-amg-highres}]
By Lemma~\ref{lem:gate-convergence-final},
\[
    \lambda_{f,\rho}
    =
    \chi_{f,\rho}^2
    \to0.
\]
By Proposition~\ref{prop:amg-complexity-convergence},
\begin{equation}
    0
    \le
    R_m(\mathfrak L^{AMG}_{f,\rho}(\lambda_{f,\rho}))
    -
    R_m(\mathfrak L^F_{f,\rho})
    \le
    C_{\rm comp}\sqrt{\lambda_{f,\rho}}
    \to0.
    \label{eq:proof-main-complexity-gap}
\end{equation}

From the finite-sample decomposition in \eqref{eq:app-U-diff-AMG},
\begin{align}
    &
    U^{AMG}_{f,\rho}(\lambda_{f,\rho};\eta)
    -
    U^F_{f,\rho}(\eta)
    \nonumber\\
    &\quad
    =
    -\Gamma^{AMG}_{f,\rho}(\lambda_{f,\rho})
    +
    4
    \left(
        R_m(\mathfrak L^{AMG}_{f,\rho}(\lambda_{f,\rho}))
        -
        R_m(\mathfrak L^F_{f,\rho})
    \right).
    \label{eq:proof-amg-uf-diff}
\end{align}
Theorem~\ref{thm:gain-convergence} gives
\[
    0
    \le
    \Gamma^{AMG}_{f,\rho}(\lambda_{f,\rho})
    \le
    \mathcal E^F_{f,\rho}
    \to0.
\]
Together with \eqref{eq:proof-main-complexity-gap}, this proves
\[
    U^{AMG}_{f,\rho}(\lambda_{f,\rho};\eta)
    -
    U^F_{f,\rho}(\eta)
    \to0.
\]

For fixed-fusion S4FFNO, \eqref{eq:app-U-diff-S4} gives
\begin{equation}
    U^{S4}_{f,\rho}(\eta)
    -
    U^F_{f,\rho}(\eta)
    =
    -\Gamma^{S4}_{f,\rho}
    +
    4
    \left(
        R_m(\mathfrak L^{S4}_{f,\rho})
        -
        R_m(\mathfrak L^F_{f,\rho})
    \right).
\end{equation}
Again, Theorem~\ref{thm:gain-convergence} gives
\[
    0
    \le
    \Gamma^{S4}_{f,\rho}
    \le
    \mathcal E^F_{f,\rho}
    \to0.
\]
If the fixed-fusion S4FFNO complexity gap satisfies
\[
    R_m(\mathfrak L^{S4}_{f,\rho})
    -
    R_m(\mathfrak L^F_{f,\rho})
    \to
    \Delta^{S4}_{m,\rho}>0,
\]
then
\[
    U^{S4}_{f,\rho}(\eta)
    -
    U^F_{f,\rho}(\eta)
    \to
    4\Delta^{S4}_{m,\rho}
    >
    0.
\]
At the same time,
\[
    U^{AMG}_{f,\rho}(\lambda_{f,\rho};\eta)
    -
    U^F_{f,\rho}(\eta)
    \to0.
\]
Therefore, for all sufficiently large $f$,
\[
    U^{AMG}_{f,\rho}(\lambda_{f,\rho};\eta)
    <
    U^{S4}_{f,\rho}(\eta).
\]
\end{proof}

\subsection{Interpretation and Scope}
\label{app:theory-scope}

The theory compares FFNO, S4FFNO, and AMGFNO as predictors defined on
the same resolved history $\mathcal U_t^{f,\rho}$. FFNO is the restricted
subclass that ignores history, S4FFNO uses fixed residual memory fusion,
and AMGFNO uses the same S4 memory branch with the adaptive memory gate.

Theorem~\ref{thm:gain-convergence} isolates the memory benefit induced
by projection-based loss of Fourier modes. At low resolution, unresolved
spatial modes can still influence the next resolved state, so the
history-conditioned Bayes target may depend on temporal history. This is
the condition in which the low-resolution preference condition in
Theorem~\ref{thm:low-resolution-preference-final} can favor memory
models. At high resolution, under the stated projection-induced partial
observation assumptions, the unresolved tail $Q_f u_t^\rho$ vanishes and
the projected Markovian target becomes sufficient. The additional
approximation gain attributable to this spectral-projection mechanism
therefore converges to zero.

Theorem~\ref{thm:main-amg-highres} gives a conditional finite-sample
comparison. Fixed-fusion S4FFNO may retain a nonvanishing effective
complexity cost, while AMGFNO drives its gate budget
$\lambda_{f,\rho}$ to zero through the frequency-aware gate. Under the
sufficient conditions of
Proposition~\ref{prop:amg-complexity-convergence}, the AMGFNO
Rademacher-complexity gap is $O(\sqrt{\lambda_{f,\rho}})$ and therefore
vanishes in the high-resolution limit.

The analysis controls one-step resolved prediction risk. It does not
directly bound accumulated autoregressive rollout error, which is
evaluated empirically in Section~\ref{sec:experiment}. The analysis also
isolates only the spectral projection mechanism. It does not rule out
persistent memory benefits from other sources of partial observability,
such as latent parameters, stochastic forcing, measurement noise,
temporal subsampling, or unresolved physical variables. This is
consistent with the Mori--Zwanzig viewpoint that unresolved variables
can induce non-Markovian closure terms \citep{Mori-Zwanzig}.



\end{document}